\documentclass{article}

\PassOptionsToPackage{numbers, compress}{natbib}

\usepackage[final]{neurips_2021}

\usepackage[utf8]{inputenc} 
\usepackage[T1]{fontenc}    
\usepackage{hyperref}       
\usepackage{url}            
\usepackage{booktabs}       
\usepackage{amsfonts}       
\usepackage{nicefrac}       
\usepackage{microtype}      
\usepackage{xcolor}         

\usepackage{booktabs}
\usepackage{hyperref}
\usepackage{nicefrac}
\usepackage{microtype}
\usepackage{amsmath,amsthm,amsfonts,amssymb,bm,accents,dsfont}
\usepackage{fancybox}
\usepackage{graphicx}
\usepackage{color}
\usepackage{array}
\usepackage{subcaption}
\usepackage{xcolor,soul,colortbl}
\usepackage{enumerate}
\usepackage{wrapfig}
\usepackage{enumitem}
\usepackage{comment}
\usepackage{makecell}
\usepackage{multirow}
\usepackage{multicol}
\usepackage{bbm}
\usepackage{algpseudocode}
\usepackage{algorithm}

\usepackage[framemethod=TikZ]{mdframed}
\mdfsetup{skipabove=8pt,
innertopmargin=5pt}

\DeclareMathOperator{\E}{\mathbb{E}}

\newcommand{\norm}[1]{\ensuremath{\left\| #1 \right\|}}
\newcommand{\abs}[1]{\ensuremath{{\left\vert #1 \right\vert}}}

\DeclareMathOperator{\conv}{conv}

\DeclareMathOperator{\sign}{sign}

\DeclareMathOperator*{\argmin}{argmin}
\DeclareMathOperator*{\argmax}{argmax}
\DeclareMathOperator*{\minimize}{minimize}

\DeclareMathOperator{\subjectto}{subject\ to}
\DeclareMathOperator{\proj}{\raisebox{-.15\baselineskip}{\Large\ensuremath{\Pi}}}

\newcommand{\calA}{\ensuremath{\mathcal{A}}}
\newcommand{\calB}{\ensuremath{\mathcal{B}}}

\newcommand{\calD}{\ensuremath{\mathcal{D}}}

\newcommand{\calF}{\ensuremath{\mathcal{F}}}
\newcommand{\calG}{\ensuremath{\mathcal{G}}}
\newcommand{\calH}{\ensuremath{\mathcal{H}}}

\newcommand{\calM}{\ensuremath{\mathcal{M}}}
\newcommand{\calN}{\ensuremath{\mathcal{N}}}

\newcommand{\calP}{\ensuremath{\mathcal{P}}}

\newcommand{\calS}{\ensuremath{\mathcal{S}}}

\newcommand{\calX}{\ensuremath{\mathcal{X}}}
\newcommand{\calY}{\ensuremath{\mathcal{Y}}}
\newcommand{\calZ}{\ensuremath{\mathcal{Z}}}

\newcommand{\bzero}{\ensuremath{\bm{0}}}

\newcommand{\bI}{\ensuremath{\bm{I}}}

\newcommand{\bSigma}{\ensuremath{\bm{\Sigma}}}

\newcommand{\bp}{\ensuremath{\bm{p}}}

\newcommand{\bs}{\ensuremath{\bm{s}}}

\newcommand{\bx}{\ensuremath{\bm{x}}}

\newcommand{\bdelta}{\ensuremath{\bm{\delta}}}

\newcommand{\blambda}{\ensuremath{\bm{\lambda}}}
\newcommand{\bmu}{\ensuremath{\bm{\mu}}}
\newcommand{\bnu}{\ensuremath{\bm{\nu}}}

\newcommand{\btheta}{\ensuremath{\bm{\theta}}}

\newcommand{\bxi}{\ensuremath{\bm{\xi}}}

\newcommand{\bbR}{\ensuremath{\mathbb{R}}}

\newcommand{\fkm}{\ensuremath{\mathfrak{m}}}

\newcommand{\fkp}{\ensuremath{\mathfrak{p}}}

\def\nd/{\textsuperscript{nd}}
\def\rd/{\textsuperscript{rd}}
\def\th/{\textsuperscript{th}}

\newcommand{\setR}{\bbR}

\makeatletter
\def\nnil{\nil}
\newcounter{prob}
\newenvironment{prob}[1][\nil]{%
	\def\tmp{#1}
	\equation
	\ifx\tmp\nnil
		\refstepcounter{prob}
		\edef\@currentlabel{\theprob}\expandafter\ltx@label\expandafter{P\Roman{prob}-counter-number}
		\tag{P\Roman{prob}}
	\else
		\tag{\tmp}
	\fi
	\aligned%
}{%
	\endaligned\endequation%
}

\newcounter{dual}

\makeatother

\newenvironment{prob*}{%
	\csname equation*\endcsname%
	\aligned%
}{%
	\endaligned%
	\csname endequation*\endcsname%
}

\newcommand{\R}{\mathbb{R}}

\renewcommand\bar\overline

\definecolor{Gray}{gray}{0.9}

\DeclareMathOperator*{\subjto}{subject\ to}

\newcommand{\bv}[1]{\mathbf{#1}}

\newcommand{\dinD}{\bdelta\in\Delta}
\newcommand{\xplusd}{\bv x + \bdelta}

\newtheoremstyle{slplain}
  {.4\baselineskip\@plus.1\baselineskip\@minus.1\baselineskip}
  {.3\baselineskip\@plus.1\baselineskip\@minus.1\baselineskip}
  {\itshape}
  {}
  {\bfseries}
  {.}
  { }
  {}
\theoremstyle{slplain} 

\newtheorem*{definition*}{Definition}
\newtheorem*{theorem*}{Theorem}
\newtheorem{theorem}{Theorem}[section]
\newtheorem{lemma}[theorem]{Lemma}
\newtheorem{proposition}[theorem]{Proposition}

\newtheorem{definition}[theorem]{Definition}

\newtheorem{assumption}[theorem]{Assumption}

\theoremstyle{definition}

\title{Adversarial Robustness with Semi-Infinite Constrained Learning}

\author{Alexander Robey$^*$ \\
University of Pennsylvania \\
\texttt{arobey1@seas.upenn.edu}
\And
Luiz F.\ O.\ Chamon$^*$ \\
University of California, Berkeley \\
\texttt{lfochamon@berkeley.edu} \\
\And
George J.\ Pappas \\
University of Pennsylvania \\
\texttt{pappasg@seas.upenn.edu} \\
\And
Hamed Hassani \\
University of Pennsylvania \\
\texttt{hassani@seas.upenn.edu}
\And 
Alejandro Ribeiro \\
University of Pennsylvania \\
\texttt{aribeiro@seas.upenn.edu} \\
}

\newcommand{\bhtheta}{\ensuremath{\bm{{\hat{\theta}^\star}}}}
\newcommand{\bhdtheta}{\ensuremath{\hat{\btheta}^\dagger}}
\newcommand{\bhmu}{\ensuremath{\bm{{\hat{\mu}^\star}}}}

\begin{document}

\maketitle

\begin{abstract}
Despite strong performance in numerous applications, the fragility of deep learning to input perturbations has raised serious questions about its use in safety-critical domains.  While adversarial training can mitigate this issue in practice, state-of-the-art methods are increasingly application-dependent, heuristic in nature, and suffer from fundamental trade-offs between nominal performance and robustness. Moreover, the problem of finding worst-case perturbations is non-convex and underparameterized, both of which engender a non-favorable optimization landscape. Thus, there is a gap between the theory and practice of adversarial training, particularly with respect to \emph{when} and \emph{why} adversarial training works.  In this paper, we take a constrained learning approach to address these questions and to provide a theoretical foundation for robust learning. In particular, we leverage semi-infinite optimization and non-convex duality theory to show that adversarial training is equivalent to a statistical problem over perturbation distributions, which we characterize completely. Notably, we show that a myriad of previous robust training techniques can be recovered for particular, sub-optimal choices of these distributions. Using these insights, we then propose a hybrid Langevin Monte Carlo approach of which several common algorithms~(e.g., PGD) are special cases. Finally, we show that our approach can mitigate the trade-off between nominal and robust performance, yielding state-of-the-art results on MNIST and CIFAR-10.  Our code is available at: \url{https://github.com/arobey1/advbench}.
\end{abstract}

{\let\thefootnote\relax\footnotetext{\hspace*{-6.0mm}$^\star$\,Alexander Robey and Luiz F.\ O.\ Chamon contributed equally to this work.}}

\vspace{-0.5em}
\section{Introduction}\label{sec:intro}

Learning is at the core of many modern information systems, with wide-ranging applications in clinical research \cite{esteva2019guide, yao2019strong, li2020domain, bashyam2020medical}, smart grids \cite{zhang2018review, karimipour2019deep, samad2017controls}, and robotics \cite{julian2020never, kober2013reinforcement, sunderhauf2018limits}. However, it has become clear that learning-based solutions suffer from a critical lack of robustness~\cite{biggio2013evasion, carlini2017towards, hendrycks2019benchmarking, djolonga2020robustness, taori2020measuring, hendrycks2020many, torralba2011unbiased}, leading to models that are vulnerable to malicious tampering and unsafe behavior~\cite{datta2014automated, kay2015unequal, angwin2016machine, sonar2020invariant, vinitsky2020robust}. While robustness has been studied in statistics for decades~\cite{tukey1960survey, huber1992robust, huber2004robust}, this issue has been exacerbated by the opacity, scale, and non-convexity of modern learning models, such as convolutional neural network~(CNNs). Indeed, the pernicious nature of these vulnerabilities has led to a rapidly-growing interest in improving the so-called \emph{adversarial robustness} of modern ML models. To this end, a great deal of empirical evidence has shown \emph{adversarial training} to be the most effective way to obtain robust classifiers, wherein models are trained on perturbed samples rather than directly on clean data~\cite{goodfellow2014explaining, madry2017towards, wong2017provable, huang2017adversarial, sinha2018gradient, shaham2018understanding}. While this approach is now ubiquitous in practice, adversarial training faces two fundamental challenges.

Firstly, it is well-known that obtaining \emph{worst-case}, adversarial perturbations of data is challenging in the context of deep neural networks~(DNNs)~\cite{carlini2019evaluating, athalye2018obfuscated}. While gradient-based methods have been shown to be empirically effective at finding perturbations that lead to misclassification, there are no guarantees that these perturbations are truly worst-case due to the non-convexity of most commonly-used ML function classes~\cite{li2019implicit}. Moreover, whereas optimizing the parameters of a DNNs is typically an overparameterized problem, finding worst-case perturbations is severely underparametrized and therefore does not enjoy the benign optimization landscape of standard training~\cite{soltanolkotabi2018theoretical, zhang2016understanding, arpit2017closer, ge2017learning, brutzkus2017globally}. For this reason, state-of-the-art adversarial attacks increasingly rely on heuristics such as random initializations, multiple restarts, pruning, and other \emph{ad hoc} training procedures~\cite{wu2020adversarial, cheng2020cat, kannan2018adversarial, guo2017countering, shaham2018defending, dhillon2018stochastic, carmon2019unlabeled, bai2019hilbert, shafahi2019adversarial, papernot2016distillation}.

The second challenge faced by adversarial training is that it engenders a fundamental trade-off between robustness and nominal performance~\cite{dobriban2020provable, javanmard2020precise, tsipras2018robustness}. In practice, penalty-based methods that incorporate clean data into the training objective are often used to overcome this issue~\cite{zhang2019theoretically, wang2019improving, zheng2016improving, ding2018mma}. However, while empirically successful, these methods cannot typically guarantee nominal or adversarial performance outside of the training samples.  Indeed, classical learning theory~\cite{vapnik2013nature,shalev2014understanding} provides generalization bounds only for the aggregated objective and not each individual penalty term. Additionally, the choice of the penalty parameter is not straightforward and depends on the underlying learning task, making it difficult to transfer across applications and highly dependent on domain expert knowledge.


\textbf{Contributions.} To summarize, there is a significant gap between the theory and practice of robust learning, particularly with respect to \emph{when} and \emph{why} adversarial training works.  In this paper, we study the algorithmic foundations of robust learning toward understanding the fundamental limits of adversarial training.  To do so, we leverage semi-infinite constrained learning theory, providing a theoretical foundation for gradient-based attacks and mitigating the issue of nominal performance degradation. In particular, our contributions are as follows:

\begin{itemize}[noitemsep,nolistsep,leftmargin=2.5em]
	\item We show that adversarial training is equivalent to a stochastic optimization problem over a specific, \emph{non-atomic} distribution, which we characterize using recent non-convex duality results~\cite{paternain2019constrained, chamon2020probably}. Further, we show that a myriad of previous adversarial attacks reduce to particular, sub-optimal choices of this distribution.
	\item We propose an algorithm to solve this problem based on stochastic optimization and Markov chain Monte Carlo. Gradient-based methods can be seen as limiting cases of this procedure.
    \item We show that our algorithm outperforms state-of-the-art baselines on MNIST and CIFAR-10.  In particular, our approach yields a ResNet-18 classifiers which simultaneously achieves greater than 50\% adversarial accuracy and greater than 85\% clean accuracy on CIFAR-10, which represents a significant improvement over the previous state-of-the-art.
	\item We provide generalization guarantees for the empirical version of this algorithm, showing how to effectively limit the nominal performance degradation of robust classifiers.
\end{itemize}

\vspace{-0.5em}
\section{Problem formulation}
\label{S:problem}

Throughout this paper, we consider a standard classification setting in which the data is distributed according to an unknown joint distribution $\calD$ over instance-label pairs $(\bv x, y)$.  In this setting, the instances $\bv x\in\calX$ are assumed to be supported on a compact subset of $\R^d$, and each label $y\in\calY := \{1, \dots, K\}$ denotes the class of a given instance $\bv x$.  By $(\Omega, \calB)$ we denote the underlying measurable space for this setting, where $\Omega = \calX\times\calY$ and $\calB$ denotes its Borel $\sigma$-algebra.  Furthermore, we assume that the joint distribution~$\calD$ admits a density~$\fkp(\bv x, y)$ defined over the sets of~$\calB$.  

At a high level, our goal is to learn a classifier which can correctly predict the label $y$ of a corresponding instance $\mathbf x$.  To this end, we let $\calH$ be a hypothesis class containing functions~$f_{\btheta}: \R^d \to \calS^K$ parameterized by vectors~$\btheta \in \Theta \subset \R^p$, where we assume that the parameter space $\Theta$ is compact and by~$\calS^K$ we denote the~$(K-1)$-simplex.  We also assume that~$f_{\btheta}(\bv x)$ is differentiable with respect to~$\btheta$ and~$\bv x$.\footnote{Note that the classes of support vector machines, logistic classifiers, and convolutional neural networks (CNNs) with softmax outputs can all be described by this formalism.} To make a prediction $\hat{y}\in\calY$, we assume that the simplex~$\calS^K$ is mapped to the set of classes~$\calY$ via $\hat{y}\in \argmax\nolimits_{k \in \calY}\ [f_{\btheta}(\bv x)]_k$ with ties broken arbitrarily.  In this way, we can think of the $k$-th output of the classifier as representing the probability that $y = k$.  Given this notation, the statistical problem of learning a classifier that accurately predicts the label $y$ of a given instance $\bv x$ drawn randomly from $\calD$ can be formulated as follows:
\begin{prob}[\textup{P-NOM}]\label{P:nominal}
    \min_{\btheta \in \Theta} \E_{(\bv x, y) \sim \calD} \Big[
    	\ell\big( f_{\btheta}(\bv x), y \big)
    \Big]
    	\text{.}
\end{prob}
Here~$\ell$ is a $[0,B]$-valued loss function and~$\ell(\cdot,y)$ is $M$-Lipschitz continuous for all~$y \in \calY$. We assume that $(\bv x, y) \mapsto \ell\big( f_{\btheta}(\bv x), y \big)$ is integrable so that the objective in~\eqref{P:nominal} is well-defined; we further assume that this map is an element of the Lebgesgue space~$L^p(\Omega, \calB, \fkp)$ for some fixed~$p\in(1,\infty)$.

\textbf{Formulating the robust training objective.}  For common choices of the hypothesis class~$\calH$, including DNNs, classifiers obtained by solving~\eqref{P:nominal} are known to be sensitive to small, norm-bounded input perturbations~\cite{tramer2020fundamental}. In other words, it is often straightforward to find a relatively small perturbations~$\bdelta$ such that the classifier correctly predicts the label $y$ of~$\bv x$, but misclassifies the perturbed sample~$\bv x + \bdelta$. This has led to increased interest in the robust analog of~\eqref{P:nominal}, namely,
\begin{prob}[\textup{P-RO}]\label{P:robust}
    P_\text{R}^\star \triangleq \min_{\btheta \in \Theta}\ \E_{(\bv x, y) \sim \calD} \left[
    	\max_{\dinD}\ \ell\big( f_{\btheta}(\xplusd), y \big)
    \right]
    	\text{.}
\end{prob}
In this optimization program, the set~$\Delta\subset\R^d$ denotes the set of valid perturbations%
\footnote{Note that~$f_{\btheta}$ must now be defined on~$\calX \oplus \Delta$, where~$\oplus$ denotes the Minkowski~(set) sum. In a slight abuse of notation, we will refer to this set as~$\calX$ from now on.}.
Typically, $\Delta$ is chosen to be a ball with respect to a given metric on Euclidean space, i.e., $\Delta = \{\bdelta \in\R^d : \norm{\bdelta} \leq \epsilon\}$.  However, in this paper we make no particular assumption on the specific form of~$\Delta$. In particular, our results apply to arbitrary perturbation sets, such as those used in~\cite{robey2020model,robey2021model,goodfellow2009measuring,wong2020learning,gowal2020achieving}.

Analyzing conditions under which~\eqref{P:robust} can be~(probably approximately) solved from data remains an active area of research. While bounds on the Rademacher complexity~\cite{awasthi2020adversarial, yin2019rademacher} and VC dimension~\cite{awasthi2020adversarial, yin2019rademacher, cullina2018pac, montasser2020efficiently, montasser2019vc} of the robust loss
\begin{align}
    \ell_\text{adv}(f_{\btheta}(\bv x), y) = \max_{\dinD} \ell\big( f_{\btheta}(\xplusd), y \big) \label{eq:rob-loss}
\end{align}
have been derived for an array of losses $\ell$ and hypothesis classes $\calH$, there are still open questions on the effectiveness and sample complexity of adversarial learning~\cite{yin2019rademacher}.  Moreover, because in general the map $\bdelta \mapsto \ell(f_{\btheta}(\bv x+\bdelta),y)$ is non-concave, evaluating the maximum in \eqref{eq:rob-loss} is not straightforward.  To this end, the most empirically effective strategy for approximating the robust loss is to leverage the differentiability of modern ML models with respect to their inputs.  More specifically, by computing gradients of such models, one can approximate the value of~\eqref{eq:rob-loss} using projected gradient ascent. For instance, in~\cite{madry2017towards,shaham2015understanding} an perturbation $\bdelta$ is computed for a fixed parameter~$\btheta$ and data point~$(\bv x, y)$ by repeatedly applying
\begin{equation}\label{E:attack_pga}
	\bdelta \gets \proj_\Delta\! \Big[
		\bdelta + \eta \sign \big[ \nabla_{\bdelta} \: \ell\big( f_{\btheta}(\bv x + \bdelta), y \big) \big]
	\Big]
		\text{,}
\end{equation}
where~$\Pi_\Delta$ denotes the projection onto~$\Delta$ and $\eta>0$ is a fixed step size.  This idea is at the heart of adversarial training, in which~\eqref{E:attack_pga} is iteratively applied to approximately evaluate the robust risk in the objective of~\eqref{P:robust}; the parameters $\btheta$ can then be optimized with respect to this robust risk. 

\textbf{Common pitfalls for adversarial training.}
Their empirical success notwithstanding, gradient-based approaches to adversarial training are not without drawbacks.  One pitfall is the fact that gradient-based algorithms are not guaranteed to provide optimal~(or even near-optimal) perturbations, since~$\bv x\mapsto \ell(f_{\btheta}(\bv x), y)$ is typically not a concave function. Because of this, heuristics~\cite{wong2020fast,shi2020adaptive} are often needed to improve the solutions obtained from~\eqref{E:attack_pga}.  Furthermore, adversarial training often degrades the performance of the model on clean data~\cite{tsipras2018robustness,raghunathan2020understanding,yang2020closer}.  In practice, penalty-based approaches are used to empirically overcome this issue~\cite{zhang2019theoretically}, although results are not guaranteed to generalize outside of the training sample. Indeed, classical learning theory guarantees generalization in terms of the aggregated objective and not in terms of the robustness requirements it may describe~\cite{vapnik2013nature,shalev2014understanding, chamon2020probably}.

In the remainder of this paper, we address these pitfalls by leveraging semi-infinite constrained learning theory. To do so, we explicitly formulate the problem of finding the most robust classifier among those that have good nominal performance. Next, we show that~\eqref{P:robust} is equivalent to a stochastic optimization problem that can be related to numerous adversarial training methods~(Section~\ref{S:dual_robust_learning}). We then provide generalization guarantees for the constrained robust learning problem when solve using empirical~(\emph{unconstrained}) risk minimization~(Section~\ref{S:csl}). Finally, we derive an algorithm based on a Langevin MCMC sampler of which~\eqref{E:attack_pga} is a particular case~(Section~\ref{S:algorithm}).
\vspace{-0.5em}
\section{Dual robust learning}
\label{S:dual_robust_learning}

In the previous section, we argued that while empirically successful, adversarial training is not without shortcomings.  In this section, we develop the theoretical foundations needed to tackle the two challenges of~\eqref{P:robust}: (a) finding worst-case perturbations, i.e., evaluating the robust loss defined in~\eqref{eq:rob-loss} and (b) mitigating the trade-off between robustness and nominal performance.  To address these challenges, we first propose the following constrained optimization problem which explicitly captures the trade-off between robustness and nominal performance:
\begin{mdframed}[roundcorner=5pt, backgroundcolor=yellow!8]
\begin{prob}[\textup{P-CON}]\label{P:main}
	P^\star \triangleq &\min_{\btheta \in \Theta} &&\E_{(\bv x, y) \sim \calD} \left[
	    	\max_{\dinD}\ \ell\big( f_{\btheta}(\xplusd), y \big)
	    \right]
    \\
	&\subjto &&\E_{(\bv x,y) \sim \calD} \left[ \ell\big(f_{\btheta}(\bv x), y\big) \right] \leq \rho
\end{prob}
\end{mdframed}
where~$\rho \geq 0$ is a desired nominal performance level.  At a high level, \eqref{P:main} seeks the most robust classifier $f_{\btheta}(\cdot)$ among those classifiers that have strong nominal performance.  In this way, \eqref{P:main} is directly designed to address the trade-off between robustness and accuracy, and as such~\eqref{P:main} will be the central object of study in this paper.  We  note that at face value, the statistical constraint in~\eqref{P:main} is challenging to enforce in practice, especially given the well-known difficulty in solving the unconstrained analog~\eqref{P:robust}.  Following~\cite{chamon2020probably,chamon2020empirical}, our approach is to use duality to obtain solutions for~\eqref{P:main} that generalize with respect to both adversarial and nominal performance.

\textbf{Computing worst-case perturbations.}
Before tackling the constrained problem~\eqref{P:main}, we begin by consider its unconstrained version, namely, \eqref{P:robust}.  We start by rewriting~\eqref{P:robust} using an epigraph formulation of the maximum function to obtain the following semi-infinite program:
\begin{prob}\label{P:semi-infinite}
	P_\text{R}^\star = &\min_{\btheta \in \Theta,\, t \in L^p}
		&&\E_{(\bv x, y) \sim \calD} \!\big[ t(\bv x,y) \big]
	\\
	&\subjto &&\ell\big( f_{\btheta}(\xplusd), y \big) \leq t(\bv x,y)
		\text{,} \quad \text{for almost every } (\bv x, \pmb\delta, y) \in \calX \times \Delta \times \calY
		\text{.}
\end{prob}
Note that~\eqref{P:semi-infinite} is indeed equivalent to~\eqref{P:robust} since
\begin{align}
    \max_{\dinD}\ell\big( f_{\btheta}(\xplusd), y \big) \leq t(\bv x,y) \iff \ell\big( f_{\btheta}(\xplusd), y \big) \leq t(\bv x,y) \:\: \text{for all } \dinD
        \text{.}
\end{align}
While at first it may seem that we have made~\eqref{P:main} more challenging to solve by transforming an unconstrained problem into an infinitely-constrained problem, notice that~\eqref{P:semi-infinite} is no longer a composite minimax problem.  Furthermore, it is \emph{linear} in~$t$, indicating that~\eqref{P:semi-infinite} should be amenable to approaches based on Lagrangian duality. Indeed, the following proposition shows that~\eqref{P:semi-infinite} can be used to obtain a statistical counterpart of~\eqref{P:robust}.

\begin{proposition}\label{T:robust_sip}
If~$(\bv x, y) \mapsto \ell\big( f_{\btheta}(\bv x), y \big) \in L^p$ for~$p\in(1,\infty)$, then~\eqref{P:robust} can be written as
\begin{prob}\label{P:primal_sip}
    P_\textup{R}^\star = \min_{\btheta \in \Theta}\ p(\btheta)
    	\text{,}
\end{prob}
for the primal function
\begin{equation}\label{E:primal_function}
	p(\btheta) \triangleq \max_{\lambda \in \calP^q}\ %
    \E_{(\bv x, y)\sim\calD} \left[
    	\E_{\bdelta \sim \lambda(\bdelta \mid \bv x, y)} \left[ \ell(f_{\btheta}(\bv x + \bdelta), y) \right]
    \right]
    	\text{,}
\end{equation}
where~$\calP^q$, with~$\frac{1}{p} + \frac{1}{q} = 1$, is the subspace of~$L^q$ containing almost everywhere non-negative functions such that~$\fkp(\bv x, y) = 0 \Rightarrow \lambda(\bdelta \mid \bv x, y) = 0$ and~$\int \lambda(\bdelta \mid \bv x, y) d\bdelta = 1$ for almost every~$(\bv x, y) \in \calX \times \calY$.
\end{proposition}

\noindent The proof is provided in Appendix~\ref{app:proof-prop-3.1}.  Informally, Proposition~\ref{T:robust_sip} shows that the robust learning problem in~\eqref{P:robust} can be recast as a problem of optimizing over a set of probability distributions $\calP^q$ taking support over $\Delta$.  This establishes an equivalence between the traditional robust learning problem~\eqref{P:robust}, where the maximum is taken over perturbations $\bdelta\in\Delta$ of the input, and its stochastic version~\eqref{P:primal_sip}, where the maximum is taken over a conditional distribution over perturbations $\bdelta\sim\lambda(\bdelta|\bv x,y)$.   Notably, a variety of training formulations can be seen as special cases of~\eqref{P:primal_sip}.  In fact, for particular sub-optimal choices of this distribution, paradigms such as random data augmentation and distributionally robust optimization can be recovered (see Appendix~\ref{app:connections} for details).

As we remarked in Section~\ref{S:problem}, for many modern function classes, the task of evaluating the adversarial loss \eqref{eq:rob-loss} is a nonconcave optimization problem, which is challenging to solve in general.  Thus, Proposition~\ref{T:robust_sip} can be seen as lifting the nonconcave inner problem $\max_{\delta\in\Delta} \ell(f_{\btheta}(\bv x),y)$ to the equivalent \emph{linear} optimization problem in~\eqref{E:primal_function} over probability distributions $\lambda\in\calP^q$.  This dichotomy parallels the one that arises in PAC vs.\ agnostic PAC learning. Indeed, while the former seeks a deterministic map~$(\btheta, \bv x, y) \mapsto \bdelta$, the latter considers instead a distribution of perturbations over~$\bdelta | \bv x, y$ parametrized by~$\btheta$. In fact, since~\eqref{P:primal_sip} is obtained from~\eqref{P:robust} through semi-infinite duality, the density of this distribution is exactly characterized by the dual variables~$\lambda$.

Note that while~\eqref{P:primal_sip} was obtained using Lagrangian duality, it can also be seen as a linear lifting of the maximization in~\eqref{eq:rob-loss}. From this perspective, while recovering~\eqref{eq:rob-loss} would require~$\lambda$ to be atomic, Proposition~\ref{T:robust_sip} shows that this is not necessary as long as~$\ell(f_{\btheta}(\bv x), y)$ is an element of~$L^p$.  That is, because $\calP^q$ does not contain any Dirac distributions, the optimal distribution $\lambda^\star$ for the maximization problem in~\eqref{E:primal_function} is \emph{non-atomic}.\footnote{Proposition~\ref{T:robust_sip} does not account for~$p \in \{1,\infty\}$ for conciseness. Nevertheless, neither of the dual spaces~${L^1}^*$ or ${L^\infty}^*$ contain Dirac distributions, meaning that for $p \in \{1,\infty\}$, $\lambda^\star$ would remain non-atomic.}  Hence, Proposition~\ref{T:robust_sip} does not show that~\eqref{P:primal_sip} finds worst-case perturbations that achieve the maximum in the objective of~\eqref{P:robust}. It does, however, show that finding worst-case perturbation is not essential to find a solution of~\eqref{P:robust}

\textbf{Exact solutions for the maximization in \texorpdfstring{\eqref{P:primal_sip}}{(PII)}.}  While~\eqref{P:primal_sip} provides a new constrained formulation for~\eqref{P:robust}, the objectives of both~\eqref{P:primal_sip} and~\eqref{P:robust} still involve the solution of a non-trivial maximization. However, whereas the maximization problem in~\eqref{P:robust} is a finite-dimensional problem which is nonconcave for most modern function classes, the maximization in~\eqref{P:primal_sip} is a linear, variational problem regardless of the function class. We can therefore leverage variational duality theory to obtain a full characterization of the optimal distribution~$\lambda^\star$ when~$p = 2$.

\begin{proposition}[Optimal distribution for~\eqref{P:primal_sip}]\label{T:lambda_star}
Let~$p = 2$~(and~$q = 2$) in Proposition~\ref{T:robust_sip} and let $[z]_+ = \max(0,z)$. For each~$(\bv x,y) \in \Omega$, there exists constants~$\gamma(\bv x,y) > 0$ and~$\mu(\bv x,y) \in \setR$ s.t.
\begin{equation}\label{E:lambda_star}
	\lambda^\star(\bdelta | \bv x, y) = \left[ \frac{\ell(f_{\btheta}(\bv x + \bdelta), y) - \mu(\bv x,y)}{\gamma(\bv x,y)} \right]_+
\end{equation}
is a solution of the maximization in~\eqref{E:primal_function}. In particular, the value of~$\mu(\bv x,y)$ is such that
\begin{equation} \label{eq:mu-and-gamma}
	\int_\Delta \left[ \ell(f_{\btheta}(\bv x + \pmb \delta), y) - \mu(\bv x,y) \right]_+ d\bdelta = \gamma(\bv x,y)
		\quad\forall (\bv x,y) \in \calX \times \calY
		\text{.}
\end{equation}
\end{proposition}

\noindent The proof is provided in Appendix~C.  This proposition shows that when~$(\bv x, y) \mapsto \ell\big( f_{\btheta}(\bv x), y \big)\in L^2$, we can obtain a closed-form expression for the distribution $\lambda^\star$ that maximizes the objective of~\eqref{P:primal_sip}.  Moreover, this distribution is proportional to a truncated version of the loss of the classifier.   Note that the assumption that the loss belongs to $L^2$ is mild given that the compactness of~$\calX$, $\calY$, and~$\Delta$ imply that~$L^{p_1} \subset L^{p_2}$ for~$p_1 > p_2$. It is, however, fundamental to obtain the closed-form solution in Proposition~\ref{T:lambda_star} since it allows~\eqref{E:primal_function} to be formulated as a strongly convex constrained problem whose primal solution~\eqref{E:lambda_star} can be recovered from its dual variables~(namely, $\gamma$ and~$\mu$).  To illustrate this result, we consider two particular \emph{suboptimal} choices for the constants $\mu$ and $\gamma$.

\noindent\textbf{Special case I: over-smoothed $\lambda^\star$.}  Consider the case when $\gamma(\bv x,y)$ is taken to be the normalizing constant $\int_\Delta \ell(f_{\btheta}(\bv x+\bdelta),y) d\bdelta$ for each $(\bv x,y)\in\Omega$.  As the loss function $\ell$ is non-negative,~\eqref{eq:mu-and-gamma} implies that $\mu(\bv x,y) = 0$, and the distribution defined in~\eqref{E:lambda_star} can be written as 
\begin{align}
    \lambda(\bdelta|\bv x,y) = \frac{\ell(f_{\btheta}(\bv x + \bdelta), y)}{\int_\Delta \ell(f_{\btheta}(\bv x + \bdelta), y)d\bdelta}, \label{eq:sampling-lambda}
\end{align}
meaning that $\lambda(\bdelta|\bv x,y)$ is exactly proportional to the loss $\ell(f_{\btheta}(\bv x+\bdelta),y)$ on a perturbed copy of the data.  Thus, for this choice of $\gamma$ and $\mu$, the distribution $\lambda$ in~\eqref{eq:sampling-lambda} is an over-smoothed version of the optimal distribution $\lambda^\star$.  In our experiments, we will use this over-smoothed approximation of $\lambda^\star$ to derive an MCMC-style sampler, which yields state-of-the-art performance on standard benchmarks.

\noindent\textbf{Special case II: under-smoothed $\lambda^\star$.}  It is also of interest to consider the case in which $\gamma$ approaches zero.  In the proof of Proposition~\ref{T:lambda_star}, we show that the value of $\mu$ is fully determined by $\gamma$ and that $\gamma$ is directly related to the smoothness of the optimal distribution; in fact, $\gamma$ is equivalent to a bound on the $L^2$ norm of $\lambda^\star$.  In this way, as we take $\gamma$ to zero, we find that $\mu$  approaches $\max_{\bdelta\in\Delta}\ell(f_{\btheta}(\mathbf x+\bm\delta),y)$, meaning that the distribution is truncated so that mass is only placed on those perturbations $\bdelta$ which induce the maximum loss.  Thus, in the limit, $\lambda$ approaches an atomic distribution concentrated entirely at a perturbation $\delta^\star$ that maximizes the loss.  Interestingly, this is the same distribution that would be needed to recover the solution to the inner maximization as in~\eqref{P:robust}.  This highlights the fact that although recovering the optimal $\delta^\star$ in~\eqref{P:main} would require $\lambda^\star$ to be atomic, the condition that $\gamma > 0$ means that $\lambda^\star$ need not be atomic.

These two cases illustrate the fundamental difference between~\eqref{P:robust} and~\eqref{P:primal_sip}: Whereas in~\eqref{P:robust} we search for worst-case perturbations, in~\eqref{P:primal_sip} we seek a method to sample perturbations $\bm\delta$ from the perturbation distribution $\lambda^\star$.  Thus, given a method for sampling $\delta\sim\lambda^\star(\bm\delta|\mathbf x,y)$, the max in~\eqref{P:robust} can be replaced by an expectation, allowing us to consider the following optimization problem:
\begin{prob}\label{P:replace-max}
    P_\textup{R}^\star = \min_{\btheta \in \Theta} \E_{(\bv x,y)\sim\calD} \left[ \E_{\bdelta\sim\lambda^\star(\bdelta|\bv x, y)} \left[ \ell(f_{\bm\theta}(\bv x+\bdelta,y) \right] \right].
\end{prob}
Notice that crucially this problem is non-composite in the sense that it no longer contains an inner maximization.  To this end, in Section~\ref{S:algorithm}, we propose a scheme that can be used to sample from a close approximation of $\lambda^\star$ toward evaluating the inner expectation in~\eqref{P:replace-max}.

\section{Solving the constrained learning problem}
\label{S:csl}

So far, we have argued that~\eqref{P:main} captures the problem of finding the most robust model with high nominal performance and we have shown the the minimax objective of~\eqref{P:main} can be rewritten as a stochastic optimization problem over perturbation distributions.  In this section, we address the distinct yet related issue of satisfying the constraint in~\eqref{P:main}, which is a challenging task in practice given the statistical and potentially non-convex nature of the problem.  Further complicating matters is the fact that by assumption we have access to the data distribution~$\calD$ only through samples~$(\bv x, y) \sim \calD$, which means that in practice we cannot evaluate either of the expectations in~\eqref{P:main}.  To overcome these obstacles, given a dataset $\{(\mathbf x_n, y_n)\}_{n=1}^N$ sampled i.i.d.\ from $\calD$, our approach is to use duality to approximate~\eqref{P:main} by the following empirical, unconstrained saddle point problem
\begin{prob}[$\widehat{\textup{DI}}$]\label{P:empirical_dual}
	\hat{D}^\star = \max_{\nu \geq 0}\ \min_{\btheta \in \Theta}\ \frac{1}{N} \sum_{n = 1}^n \left[
	   	\max_{\dinD}\ \ell\big( f_{\btheta}(\bv x_n + \bdelta), y_n \big)
			+ \nu \left[ \ell\big( f_{\btheta}(\bv x_n), y_n \big) - \rho \right]
	\right]
		\text{.}
\end{prob}
Conditions under which solutions of~\eqref{P:empirical_dual} are~(probably approximately) near-optimal \emph{and} near-feasible for~\eqref{P:main} were obtained in~\cite{chamon2020probably}. As one would expect, these guarantees only hold when the objective and constraint of~\eqref{P:main} are learnable individually. As we discussed in Section~\ref{S:problem}, this is known to hold in a variety of scenarios~(e.g., when the Rademacher complexity or VC-dimension is bounded), although obtaining more general results remains an area of active research~\cite{awasthi2020adversarial, yin2019rademacher, cullina2018pac, montasser2020efficiently, montasser2019vc}. In what follows, we formalize these generalization results in our setting, starting with the learning theoretic assumptions we require on the objective and constraint.

\textbf{Learning theoretic assumptions for \texorpdfstring{\eqref{P:main}}{(P-CON)}.}
We first assume that the parameterization space $\Theta$ is sufficiently expressive (Assumption~\ref{A:parametrization} and that there exists parameters $\btheta\in\Theta$ that strictly satisfy the nominal performance constraint (Assumption~\ref{A:slater}).  We also assume that uniform convergence holds for the objective and constraint (Assumption~\ref{A:empirical}).
\begin{assumption}\label{A:parametrization}
The parametrization~$f_{\btheta}$ is rich enough so that for each~$\btheta_1, \btheta_2 \in \Theta$ and~$\beta \in [0,1]$, there exists~$\btheta \in \Theta$ such that $\sup_{\bv x \in \calX}\ \abs{\beta f_{\btheta_1}(\bx) + (1-\beta) f_{\btheta_2}(\bx)
		- f_{\btheta}(\bx)} \leq \alpha
		\text{.}$
\end{assumption}

\begin{assumption}\label{A:slater}
There exists~$\btheta^\prime \in \Theta$ such that~$\E_{\calD} \left[ \ell\big(f_{\btheta^\prime}(\bv x), y\big) \right] < \rho - M\alpha$.
\end{assumption}

\begin{assumption}\label{A:empirical}
	There exists~$\zeta_R(N), \zeta_N(N) \geq 0$ monotonically decreasing in~$N$ such that~$\forall\btheta \in \Theta$:\vspace{-0.3em}
	\begin{subequations}\label{E:emp_approximation}
	\begin{align}
		\abs{\E_{(\bx,y) \sim \calD} \!\big[ \max_{\dinD}\ \ell\big( f_{\btheta}(\xplusd), y \big) \big]
			- \frac{1}{N} \sum_{n = 1}^{N} \max_{\dinD}\ \ell\big( f_{\btheta}(\bx_n +\bdelta), y_n \big)} &\leq \zeta_R(N)
			\text{ w.p. } 1-\delta
			\label{E:emp_approximation_robust}
		\\
		\abs{\E_{(\bx,y) \sim \calD} \!\big[ \ell(f_{\btheta}(\bv x), y) \big]
			- \frac{1}{N} \sum_{n = 1}^{N} \ell(f_{\btheta}(\bv x_n), y_n)} &\leq \zeta_N(N)
			\text{ w.p. } 1-\delta
			\label{E:emp_approximation_nominal}
	\end{align}
	\end{subequations}
\end{assumption}

\noindent One natural question to ask is whether the bounds in~\eqref{E:emp_approximation_robust} and~\eqref{E:emp_approximation_nominal} hold in practice.  We note that in the non-adversarial bound~\eqref{E:emp_approximation_nominal} has been shown to hold for a wide variety of hypothesis classes, including DNNs~\cite{bartlett2017spectrally,shalev2014understanding}.  And although these classical results do not imply the robust uniform convergence property~\eqref{E:emp_approximation_robust}, there is a growing body of evidence which suggests that this property does in fact hold for the function class of DNNs~\cite{yin2019rademacher,khim2018adversarial,tu2019theoretical}.

\textbf{Near-optimality and near-feasibility of \texorpdfstring{\eqref{P:empirical_dual}}{(DI)}.}
By combining these assumptions with the techniques used in~\cite{chamon2020probably}, we can explicitly bound the empirical duality gap (with high probability) and characterize the feasibility of the empircal dual optimal solution for~\eqref{P:main}.
\begin{proposition}[The empirical dual of~\eqref{P:main}]\label{T:dual}
Let~$\ell(\cdot,y)$ be a convex function for all~$y \in \calY$. Under Assumptions~\ref{A:parametrization}--\ref{A:empirical}, it holds with probability~$1-5\delta$ that

\begin{enumerate}
	\item $\big|P^\star - \hat{D}^\star\big| \leq M \alpha + (1 + \bar{\nu}) \max(\zeta_R(N), \zeta_N(N))$; and
	\item There exists~$\btheta^\dagger \in \argmin_{\btheta \in \Theta} \hat{L}(\btheta, \hat{\nu}^\star)$ such that~$\E_{(\bv x,y) \sim \calD} \left[ \ell\big(f_{\btheta^\dagger}(\bv x), y\big) \right] \leq \rho + \zeta_N(N)$.
\end{enumerate}
\noindent Here, $\hat{\nu}^\star$ denotes a solution of~\eqref{P:empirical_dual}, $\nu^\star$ denotes an optimal dual variable of~\eqref{P:main} solved over~$\bar\calH = \conv(\calH)$ instead of~$\calH$, and~$\bar{\nu} = \max(\hat{\nu}^\star,\nu^\star)$. Additionally, for any interpolating classifier~$\btheta^\prime$, i.e.\ such that $\E_{(\bv x,y) \sim \calD} \left[ \ell\big(f_{\btheta^\prime}(\bv x), y\big) \right] = 0$, it holds that
\begin{equation}\label{E:nu_bound}
	\nu^\star \leq \rho^{-1} \E_{(\bv x, y) \sim \calD} \Big[
    	\max_{\dinD}\ \ell\big( f_{\btheta^\prime}(\xplusd), y \big)
    \Big]
    	\text{.}
\end{equation}
\end{proposition}
\noindent The proof is provided in Appendix~D.  At a high level, Proposition~\ref{T:dual} tells us that it is possible to learn robust models with high clean accuracy using the empirical dual problem in~\eqref{P:empirical_dual} at little cost to the sample complexity.  This means that seeking a robust classifier with a given nominal performance is~(probably approximately) equivalent to seeking a classifier that minimizes a combination of the nominal and adversarial empirical loss.  Notably, the majority of past approaches for solving~\eqref{P:main} cannot be endowed with similar guarantees in the spirit of Proposition~\ref{T:dual}.  Indeed, while the objective resembles a penalty-based formulation, notice that~$\nu$ is an \emph{optimization variable} rather than a fixed hyperparameter.  Concretely, the magnitude of this dual variable $\nu$ quantifies how hard it is to learn an adversarially robust model while maintaining strong nominal performance.  Though seemingly innocuous, this caveat is the difference between guaranteeing generalization only on the aggregated loss and guaranteeing generalization jointly for the objective value and constraint feasibility.

\vspace{-0.5em}
\begin{algorithm}[t!]
    \centering
    \begin{algorithmic}[1]
		\Statex Initialize $\btheta \gets \btheta_0$ and~$\nu \gets 0$

		\Repeat
	        \For{Batch $\{(\bv x_i, y_i)\}_{i=1}^m$}
				\State $\bdelta_i \gets \bv 0$, for $i = 1,\dots,m$

				\For{$L$ steps}
					\State $\displaystyle
						U_{i} \gets \log\Big[ \ell_\text{pert}\big( f_{\btheta}(\bv x_i + \bdelta_i ), y_i \big) \Big]$,
					for $i = 1,\dots,m$

					\State $\displaystyle
						\bdelta_{i} \gets \proj_\Delta \Big[
							\bdelta_{i} + \eta \nabla_{\bdelta_i} U_{i} + \sqrt{2\eta T} \bxi_i
						\Big]$,
					where~$\bxi_i \sim \text{Laplace}(0,I)$ and $i = 1,\dots,m$
				\EndFor

				\State $\displaystyle
					\btheta \gets \btheta - \frac{\eta_p}{m}\sum_{i=1}^m \nabla_{\btheta}
						\big[ \ell_\text{ro}\big( f_{\btheta}(\bv x_i + \bdelta_{i}), y_i \big)
							+ \nu \ell_\text{nom}\big( f_{\btheta}(\bv x_i), y_i \big) \big]$
			\EndFor

			\State $\displaystyle
				\nu \gets \left[ \nu + \eta_d
					\left( \frac{1}{N} \sum_{n = 1}^N \ell\big( f_{\btheta}(\bv x_n), y_n \big) - \rho \right)
				\right]_+$
		\Until{convergence}
    \end{algorithmic}
    \caption{Semi-Infinite Dual Adversarial Learning (DALE)}
    \label{L:algorithm}
\end{algorithm}

\section{Dual robust learning algorithm}
\label{S:algorithm}\vspace{-0.2em}

Under the mild assumption that $(\bv x,y)\mapsto \ell(f_{\btheta}(\bv x),y) \in L^2$, Propositions~\ref{T:robust_sip}, \ref{T:lambda_star}, and \ref{T:dual} allow us to transform~\eqref{P:main} into the following \textbf{D}ual \textbf{A}dversarial \textbf{LE}arning problem
\begin{prob}[\textup{P-DALE}]\label{P:equivalent}
	\hat{D}^\star \triangleq \max_{\nu \geq 0}\ \min_{\btheta \in \Theta}\ %
		\frac{1}{N} \sum_{n = 1}^n \Big[
	    	\E_{\bdelta_n} \left[ \ell(f_{\btheta}(\bv x_n + \bdelta_n), y_n) \right]
				+ \nu \left[ \ell\big( f_{\btheta}(\bv x_n), y_n \big) - \rho \right]
		\Big]
\end{prob}
where~$\bdelta_n \sim \lambda_n^\star := \gamma_n^{-1} \big[ \ell(f_{\btheta}(\bv x_n + \bdelta_n), y_n) - \mu_n \big]_+$ for each $n\in\{1, \dots, N\}$ and~$\gamma_n > 0$ and~$\mu_n$ are the constants specified in Proposition~\ref{T:lambda_star}.  Note that this formulation is considerably more amenable than~\eqref{P:main}. Indeed, it is (i)~empirical and therefore does not involve unknown statistical quantities such as~$\calD$; (ii)~unconstrained and therefore more amendable to gradient-based optimization techniques; and (iii)~its objective does not involve a challenging maximization problem in view of the closed-form characterization of~$\lambda^\star$ in Proposition~\ref{T:lambda_star}.  In fact, for models that are linear in~$\btheta$ but nonlinear in the input~(e.g., kernel models or logistic regression), this implies that we can transform a non-convex, composite optimization problem~\eqref{P:main} into a convex problem~\eqref{P:equivalent}. 

Nevertheless, for many modern ML models such as CNNs, \eqref{P:equivalent} remains a non-convex program in~$\btheta$. And while there is overwhelming theoretical and empirical evidence that stochastic gradient-based algorithms yield good local minimizers for such overparametrized problems~\cite{soltanolkotabi2018theoretical, zhang2016understanding, arpit2017closer, ge2017learning, brutzkus2017globally}, the fact remains that solving~\eqref{P:equivalent} requires us to evaluate an expectation with respect to~$\lambda^\star$, which is challenging due to the fact that $\mu_n$ and $\gamma_n$ are not known a priori.  In the remainder of this section, we propose a practical algorithm to solve~\eqref{P:equivalent} based on the approximation discussed in Section~\ref{S:dual_robust_learning}.

\textbf{Sampling from the optimal distribution \texorpdfstring{$\lambda^\star_n$}{}.}
Although Proposition~\ref{T:lambda_star} provides a characterization of the optimal distribution $\lambda_n^\star$, obtaining samples from $\lambda_n^\star$ can still be challenging in practice, especially when the dimension of~$\bdelta_n$ is large~(e.g., for image-classification tasks).  Moreover, in practice the value of~$\gamma$ for which~\eqref{E:lambda_star} is a solution of~\eqref{E:primal_function} is not known \emph{a priori} and can be arbitrarily close to zero, making $\lambda^\star_n$ discontinuous and with a potentially vanishing support.  Fortunately, these issues can be addressed by using Hamiltonian Monte Carlo~(HMC) methods, which leverage the geometry of the distribution to overcome the curse of dimensionality.  

In particular, we propose to use a projected Langevin Monte Carlo~(LMC) sampler~\cite{bubeck2015finite}.  To derive this sampler, we first make a simplifying assumption: Rather than seeking the optimal constants $\gamma$ and $\mu$, we consider the over-smoothed approximation of $\lambda_n^\star$ derived in~\eqref{eq:sampling-lambda}, wherein the probability mass allocated to a particular perturbation $\bdelta\in\Delta$ is proportional to the loss $\ell(f_{\btheta}(\bv x+\bdelta),y)$.  We note that while this choice of $\lambda_n^\star$ may not be optimal, the sampling scheme that we derive under this assumption yields strong numerical performance.  Furthermore, even if we knew the true values of $\gamma_n$ and $\mu_n$, the resulting distribution for $\mu_n\neq 0$ would be discontinuous and sampling from such distributions in high-dimensional settings is challenging in and of itself (see, e.g., \cite{nishimura2020discontinuous}).

\begin{figure}
    \centering
    \includegraphics[width=\textwidth]{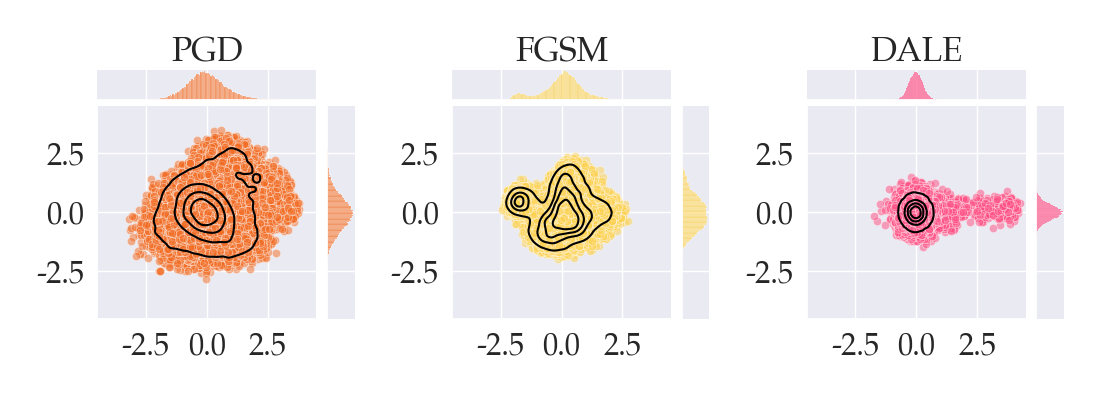}
    \caption{\textbf{Visualizing the distribution of adversarial perturbations.}  In this figure, we visualize the distribution of adversarial perturbations by projecting the perturbations generated by PGD, FGSM, and DALE onto their first two principal components.  The first and second principal components are shown on the $x$- and $y$-axes respectively.  Notice that DALE varies much less along the second principal component vis-a-vis PGD and FGSM; this indicates that DALE tends to focus more on directions in which the data varies most, indicating that it finds stronger adversarial perturbations.}
    \label{fig:pca-mnist}
\end{figure}

Given this approximate characterization of the optimal distribution, the following Langevin iteration can be derived directly from the commonly-used leapfrog simpletic integrator for the Hamiltonian dynamics induced by the distribution $\lambda_n$ defined in~\eqref{eq:sampling-lambda} (see Appendix~\ref{app:sampler} for details).  This, in turn, yields the following update rule:
\begin{align}
    \bdelta &\gets \Pi_{\Delta} \Big[
							\bdelta + \eta \sign\left[ \nabla_{\bdelta} \log \Big[ \ell_\text{pert}(f_{\btheta}(\bv x + \bdelta), y)\Big] + \sqrt{2\eta T} \bxi
						\right]\Big]
\end{align}
where~$\bxi \sim \textup{Laplace}(\bv 0, \bI)$.  In this notation, $T>0$ and $\eta > 0$ are constants which can be chosen as hyperparameters, and $\ell_\text{pert}$ is a loss functions for the perturbation.  The resulting algorithm is summarized in Algorithm~\ref{L:algorithm}.  Notice that Algorithm~\ref{L:algorithm} accounts for scenarios in which the losses associated with the adversarial performance~($\ell_\text{ro}$), the perturbation~($\ell_\text{pert}$), and the nominal performance~($\ell_\text{nom}$) are different. It can therefore learn from perturbations that are adversarial for a different loss than the one used for training the model~$\btheta$. This generality allows it to tackle different applications, e.g., by replacing the adversarial error objective in~\eqref{P:main} by a measure of model invariance~(e.g., ACE in~\cite{zhang2019theoretically}). This feature can also be used to show that existing adversarial training procedures can be seen as approximations of Algorithm~\ref{L:algorithm}~(see~Appendix~A).  We refer the reader to Appendix~\ref{app:conv} for further discussion of the convergence properties of Algorithm~\ref{L:algorithm}.

\vspace{-0.5em}

\section{Experiments} \label{sect:experiments}

In this section, we include an empirical evaluation of the DALE algorithm.  In particular, we consider two standard datasets: MNIST and CIFAR-10.  For MNIST, we train four-layer CNNs and set $\Delta = \{\bdelta : \norm{\bdelta}_\infty \leq 0.3\}$; for CIFAR-10, we train ResNet-18 models and set $\Delta = \{\bdelta : \norm{\bdelta}_\infty \leq 8/255\}$.  All hyperparameters and performance metrics are chosen with respect to the robust accuracy of a PGD adversary evaluated on a small hold-out validation set.  Further details concerning hyperparameters and architectures are provided in Appendix~\ref{sect:hyperparams}.  We also provide additional experiments in Appendix~\ref{app:further-exps}.

\textbf{Evaluating the adversarial robustness of DALE.}  We begin our empirical evaluation by comparing the adversarial robustness of DALE with numerous state-of-the-art baselines in Table~\ref{tab:mnist-and-cifar-linf}.  To evaluate the robust performance of these classifiers, we use a 1-step and an $L$-step PGD adversary to evaluate robust performance; we denote these adversaries by FGSM and PGD$^{L}$ respectively.  Notice that on CIFAR-10, DALE with $\rho=0.8$ is the only method to achieve higher than 85\% clean accuracy and 50\% adversarial accuracy against PGD$^{20}$.  Furthermore, when DALE is run with $\rho=1.1$, we see that it achieves nearly 52\% adversarial accuracy, which is a significant improvement over all baselines.  In Appendix~\ref{app:further-exps}, we provide a more complete characterization of the role of $\rho$ in controlling the trade-off between robustness and accuracy.

\begin{table}
\centering
\caption{\textbf{Adversarial robustness on MNIST and CIFAR-10.}  Test accuracies of DALE (Algorithm~\ref{L:algorithm}) and state-of-the-art baselines on MNIST and CIFAR-10.  On both datasets, DALE surpasses the baselines against both adversaries, while simultaneously maintaining high nominal performance.} \vspace{5pt}
\label{tab:mnist-and-cifar-linf}
\begin{tabular}{ccccccccc}
\toprule
& & \multicolumn{3}{c}{MNIST} & \multicolumn{3}{c}{CIFAR-10} \\ \cmidrule(lr){3-5} \cmidrule(lr){6-8}
\textbf{Algorithm} & $\bm\rho$ & \textbf{Clean} & \textbf{FGSM} & \textbf{PGD$^{10}$} & \textbf{Clean} & \textbf{FGSM} & \textbf{PGD$^{20}$} \\
\midrule
ERM & - & 99.3 & 14.3 & 1.46 & 94.0 & 0.01 & 0.01 \\ \midrule
FGSM & - & 98.3 & 98.1 & 13.0 & 72.6 & 49.7 & 40.7  \\
PGD & - & 98.1 & 95.5 & 93.1 & 83.8 & 53.7 & 48.1 \\
CLP & - & 98.0 & 95.4 & 92.2 & 79.8 & 53.9 & 48.4 \\
ALP & - & 98.1 & 95.5 & 92.5 & 75.9 & 55.0 & 48.8 \\
TRADES & - & 98.9 & 96.5 & 94.0 & 80.7 & 55.2 & 49.6 \\
MART & - & 98.9 & 96.1 & 93.5 & 78.9 & 55.6 & 49.8 \\
\midrule
\rowcolor{Gray} DALE & 0.5 & 99.3 & 96.6 & 94.0 & 86.0 & 54.4 & 48.4 \\
\rowcolor{Gray} DALE & 0.8 & 99.0 & 96.9 & 94.3 & 85.0 & 55.4 & 50.1 \\
\rowcolor{Gray} DALE & 1.0 & 99.1 & 97.7 & 94.5 & 82.1 & 55.2 & 51.7 \\
\bottomrule
\end{tabular}
\end{table}

\textbf{Visualizing the distribution of adversarial perturbations.}  To visualize the distribution over perturbations generated by DALE, we use principal component analysis~(PCA) to embed perturbations into a two-dimensional space.  In particular, we performed PCA on the MNIST training set to extract the first two principal components of the images; we then projected the perturbations~$\bdelta\in\Delta$ generated by PGD, FGSM, and DALE in the last iteration of training onto these principal components.  A plot of these projections is shown in Figure~\ref{fig:pca-mnist}, in which the first and second principal components are shown on the $x$- and $y$-axes respectively.  Notice that the perturbations generated by FGSM are spread out somewhat unevenly in this space. In contrast, the perturbations found by PGD and DALE are spread out more evenly. Furthermore, the perturbations generated by PGD and FGSM vary more along the second principal component~($y$-axis) than the first~($x$-axis) relative to DALE. Since the first component describes the direction of largest variance of the data, this indicates that DALE tends to find perturbations that place more mass on the direction in which the data varies most.

\begin{figure}
    \centering
    \includegraphics[width=\textwidth]{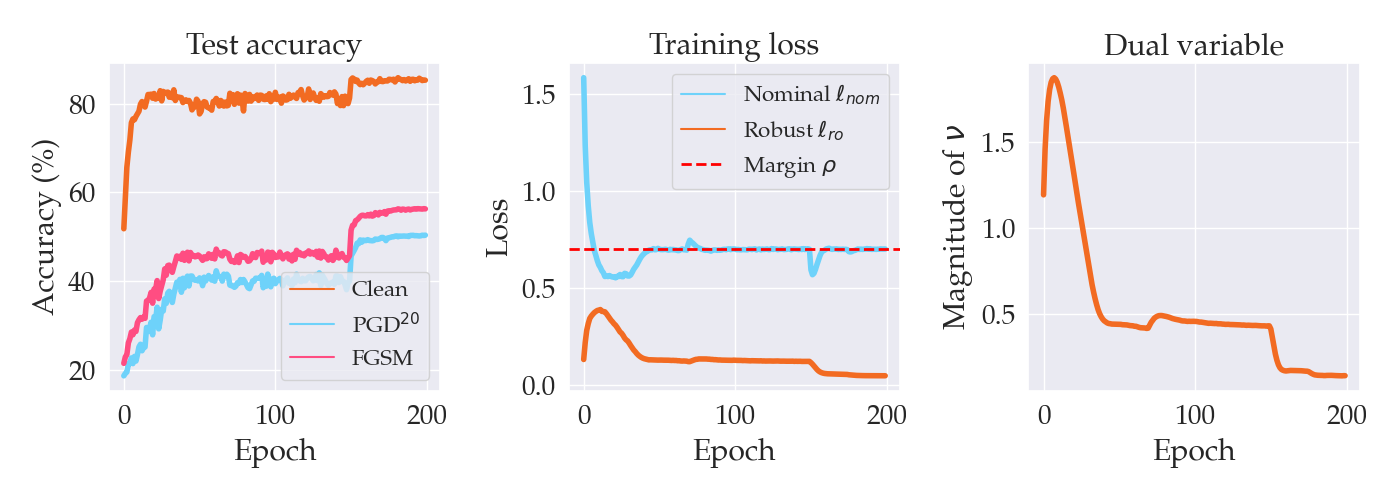}
    \caption{\textbf{Tracking the dual variables.}  Left: the clean and robust test accuracies of a ResNet-18 classifier trained on CIFAR-10 using DALE.  Middle: the training losses for DALE.  Right: The magnitude of the dual variable during training.}
    \label{fig:primal-dual-tracking}
\end{figure}

\textbf{Tracking the dual variables.}  In Figure~\ref{fig:primal-dual-tracking} we study the performance of DALE over the course of training.  In the leftmost panel, we plot the test accuracy on clean samples and on adversarially perturbed samples.  Notably, this classifier exceeds 50\% robust accuracy against PGD$^{20}$ as well as 85\% clean accuracy; these figures are higher than either of the corresponding metrics for any of the baselines in Table~\ref{tab:mnist-and-cifar-linf}, indicating that our method is more effectively able to mitigate the trade-off between robustness and accuracy.  In the middle panel of Figure~\ref{fig:primal-dual-tracking}, we track the nominal and robust training losses, and in the rightmost panel, we show the magnitude of the dual variable $\nu$.  Observe that at the onset of training, the constraint in~\eqref{P:main} is not satisfied, as the blue curve is above the red dashed-line.  In response, the dual variable places more weight on the nominal loss term in~\eqref{P:equivalent}.  After several epochs, this reweighting forces constraint satisfaction, after which the dual variable begins to decrease, which in turn decreases the weight on the nominal objective and allows the optimizer to focus on minimizing the robust loss.

\begin{wraptable}{r}{5.2cm}
    \centering
    \caption{\textbf{Regularized DALE.}  Test accuracies attained by running DALE without the dual-update step in line~10 of Algorithm~\ref{L:algorithm}.}\vspace{5pt}
    \label{tab:regularization}
    \begin{tabular}{cccc} \toprule
         $\bnu$ & \textbf{Clean} & \textbf{FGSM} & \textbf{PGD$^{20}$} \\ \midrule
         0.1 & 86.4 & 55.3 & 49.5 \\ 
         0.2 & 86.8 & 54.2 & 49.3 \\
         0.3 & 86.3 & 54.8 & 48.2 \\
         0.4 & 86.2 & 54.6 & 47.3 \\
         0.5 & 86.5 & 54.3 & 46.8 \\
         0.6 & 85.7 & 53.3 & 46.4 \\
         0.7 & 85.8 & 53.3 & 46.0 \\
         0.8 & 84.9 & 53.1 & 45.9 \\
         0.9 & 85.0 & 53.4 & 45.7 \\
         1.0 & 84.5 & 52.7 & 45.8 \\ \bottomrule
    \end{tabular}
    \vspace{-10mm}
\end{wraptable}

\textbf{Regularization vs. primal-dual.}  Our final ablation study is to consider the impact of performing the dual-update step in line~10 of Algorithm~\ref{L:algorithm}.  In particular, in Table~\ref{tab:regularization}, we record the performance of DALE when Algorithm~\ref{L:algorithm} is run without the dual update step.  This corresponds to running DALE with a fixed weight $\nu$. Notice that although our method reaches the same level of robust performance as MART and TRADES, it does not match the performance of the DALE classifiers in Table~\ref{tab:mnist-and-cifar-linf}.  This indicates that the strong robust performance of our algorithm relies on adaptively updating the dual variable over the course of training. 

\vspace{-0.5em}
\section{Conclusion}

In this paper, we studied robust learning from a constrained learning perspective.  We proved an equivalence between the standard adversarial training paradigm and a stochastic optimization problem over a specific, non-atomic distribution.  This insight provided a new perspective on robust learning and engendered a Langevin MCMC approach for adversarial robustness.  We experimentally validated that this algorithm outperforms the state-of-the-art on standard benchmarks.  Notably, our method simultaneously achieved greater than 50\% adversarial accuracy and greater than 85\% clean accuracy on CIFAR-10, which represents a significant improvement over the previous state-of-the-art.

\newpage
\section{Acknowledgements and disclosure of funding}

The authors would like to thank Juan Cervino and Samuel Sokota for their helpful feedback.

The research of Alexander Robey and Hamed Hassani is supported by NSF Grants 1837253, 1943064, 1934876, AFOSR Grant FA9550-20-1-0111, and DCIST-CRA.  Alexander Robey, George J. Pappas, Luiz F.\ O.\ Chamon, and Alejandro Ribeiro are all supported by the NSF-Simons Program on the Mathematics of Deep Learning.  Alexander Robey and George J. Pappas are supported by ARL CRA DCIST W911NF-17-2-0181 and NSF 2038873.

\newpage
\bibliography{bibliography}
\bibliographystyle{unsrt}

\newpage
\appendix
\newpage

\section{Connections to other problems} \label{app:connections}

In the Introduction, we claimed that several well-known formulations for robust learning can be recovered as approximations of~\eqref{P:primal_sip} or Algorithm~\ref{L:algorithm}. In what follows, we explore these connections in three directions: (i)~fixed, sub-optimal choices of the perturbation distribution~$\lambda$ in~\eqref{E:primal_function}; (ii)~limiting cases of the projected LMC dynamics used in Algorithm~\ref{L:algorithm}; (iii)~modifications of the empirical dual problem~\eqref{P:empirical_dual}.

\subsection{Sub-optimal perturbation distributions}

For a fixed perturbation distribution~$\lambda$ in~\eqref{E:primal_function}, \eqref{P:primal_sip} can be thought of as a random data augmentation procedure~\cite{holmstrom1992using}. For instance, for~$\lambda = \calN(\bzero, \bSigma)$, \eqref{P:primal_sip} becomes
\begin{equation*}
    \minimize_{\btheta\in\Theta}\ \E_{(\bv x,y) \sim \calD}
    	\Big[ \E_{\bdelta\sim\calN(\bzero, \bSigma)} \big[ \ell\big( f_{\btheta}(\bv x + \bdelta), y \big) \big] \Big]
    		\text{.}
\end{equation*}
Indeed, the authors of~\cite{lopes2019improving, rusak2020simple} suggest that Gaussian data augmentation can significantly improve generalization, particularly the augmentations are applied using patches~\cite{lopes2019improving}. To quote from \cite{rusak2020simple}:
\begin{quote}
    Data augmentation with Gaussian \dots noise serves as a simple yet very strong baseline that is sufficient to surpass almost all previously proposed defenses against common corruptions.
\end{quote}
More complex choices of distributions lead to other robust learning methods involving random removal of patches from the image~(e.g., \texttt{Cutout}~\cite{devries2017improved, zhong2020random}), random replacement of patches~(e.g., \texttt{CutMix}~\cite{yun2019cutmix, takahashi2019data}), or arbitrary generative models, i.e., $\bdelta \sim G \:\#\: \calN(\bzero,\sigma^2 \bI)$ for some measurable~$G: \R^z \to \R^d$~\cite{rusak2020simple}. For generative models with latent dimension~$z \ll d$, the latter approach can be thought of as parameterizing the perturbation distribution~$\lambda$ on a lower-dimensional manifold in the data space, which has been shown to be a strong defense in prior work~\cite{jalal2017robust,xiao2018generating, samangouei2018defense}.

While data augmentation with random noise has been shown to be an effective method for improving robustness in practice, the results in this paper show that even larger gains are possible by optimizing over the perturbation-generating distribution. In particular, Proposition~\ref{T:lambda_star} establishes that the optimal perturbation distribution is not Gaussian and most importantly, not isotropic. Indeed, Figure~\ref{fig:pca-mnist} suggests that the pertubation distribution arising from Algorithm~\ref{L:algorithm} does resemble an anisotropic Gaussian, but only on the basis induced by the principal components of the data.

It is worth noting that, as we mentioned in Section~\ref{S:problem}, the results of this paper do not rely on the linearity of the perturbations. Hence, more complex pertubations can be considered by using an arbitrary, parametrized data transformation~$G: \calX \times \Delta \to \calX$ as in
\begin{prob}\label{P:robust_transformed}
    \minimize_{\btheta \in \Theta}\ \E_{(\bv x, y)\sim\calD} \left[
        	\E_{\bdelta \sim \lambda} \Big[ \ell\big( f_{\btheta}\big( G(\bv x, \bdelta) \big), y \big) \Big]
        \right].
\end{prob}
Due to space constraints, we considered only pertubations of the form~$G(\bv x, \bdelta) = \bv x + \bdelta$ as in~\eqref{P:robust}. Yet, by once again fixing the pertubation distribution~$\lambda$, we can obtain a myriad of data augmentation techniques, including the group-theoretic data-augmentation scheme discussed in~\cite{chen2020group}, where~$G$ denotes the group action, and the model-based robust training methods discussed in~\cite{robey2020model,robey2021model,goodfellow2009measuring,wong2020learning}.  Indeed, exploring the efficacy of DALE toward improving robustness beyond norm-bounded perturbations is an exciting direction for future work.

\subsection{Sampling vs.\ optimizing pertubations}

Aside from fixing the pertubation distribution~$\lambda$, another common approach to adversarial learning is to use a gradient-based local optimization method in order to tackle the maximization in $\ell_\text{adv}$ (see e.g., \cite{goodfellow2014explaining,madry2017towards}). The perturbations found by these gradient-based methods can then be used to train a robust model.  While empirically effective, this approach is not without issues. In particular, gradient-based algorithms are not guaranteed to obtain optimal~(or even near-optimal) perturbations, since~$\ell(f_{\btheta}(\cdot), y)$ is typically not a convex (or concave) function. What is more, maximizing over~$\bdelta$ in the definition of~$\ell_\text{adv}$ is a severely underparametrized problem as opposed to the minimization over~$\btheta$ in~\eqref{P:robust}. It therefore does not enjoy the same benign optimization landscape~\cite{soltanolkotabi2018theoretical, zhang2016understanding, arpit2017closer, ge2017learning, brutzkus2017globally}. Additionally, note that there is no guarantee that this alternating optimization technique converges.

Nevertheless, these algorithms can be seen as limiting cases of Algorithm~\ref{L:algorithm} for specific choices of the losses~$\ell_\text{pert}$, $\ell_\text{ro}$, and $\ell_\text{nom}$, the LMC kinetic energy~(step~6), and the temperature~($T$ in step~5). To illustrate this idea, suppose that both~$\ell_\text{pert}$ and~$\ell_\text{ro}$ are taken to be the cross-entropy loss, i.e., 
\begin{align}
    \ell\big( f_{\btheta}(\bv x), y \big) = -\log \big( [f_{\btheta}(\bv x)]_y \big).
\end{align}
In this case, when we take $T\to0$, Algorithm~\ref{L:algorithm} approaches the gradient-based attacks FGSM (for $L=1$)~\cite{goodfellow2014explaining} and PGD (for $L>1$)~\cite{madry2017towards}.  However, as we observed in Figure~\ref{fig:pca-mnist}, these methods can produce quite different perturbations compared to the perturbations produced by DALE.

Another interesting perspective on gradient-based methods is to consider a different sampling scheme.  Indeed, while we adopted the commonly used Laplacian LMC sampler in Algorithm~\ref{L:algorithm}, an alternative often used to sample from lighter tailed distributions is Gaussian LMC~\cite{neal2011mcmc, bubeck2015finite, nishimura2020discontinuous}. In the latter, rather than defining the kinetic energy of the Hamiltonian as~$K(\bp) \propto \norm{\bp}_1$, this prior is taken to be~$K(\bp) \propto \norm{\bp}_2^2$. This is equivalent to replacing steps~5 and~6 of Algorithm~\ref{L:algorithm} by
\begin{align}\label{E:laplacian_lmc}
    U &\gets  \log \Big[ \ell_\text{pert}(f_{\btheta}(\bv x + \bdelta), y)\Big] \\
    \bdelta &\gets \Pi_{\Delta} \Big[
							\bdelta +  \nabla_{\bdelta} U + \sqrt{2\eta T} \bxi \Big]
\end{align}
where~$\bxi \sim \calN(\bv 0, \bI)$. An interesting direction for future work is to compare the performance of HMC-based samplers under different priors.  For more details regarding the derivation of our LMC sampler, see Appendix~\ref{app:sampler}.

\subsection{Penalty-based methods}

The third approximation of Algorithm~\ref{L:algorithm}, or more precisely, \eqref{P:empirical_dual}, is the use of a fixed~$\nu > 0$, e.g.,~\cite{zhang2019theoretically, wang2019improving, croce2020robustbench}. Indeed, notice that in the definition of $\hat{L}$, $\nu$ is an optimization variable that is dynamically adjusted in Algorithm~\ref{L:algorithm} through the dual ascent update in step~10. Notice that step~10 is simply a (sub)gradient ascent update given that the constraint violation is a subgradient of the dual function~$\hat{d}(\nu) = \min_{\btheta\in\Theta}\ \hat{L}(\btheta,\nu)$ for the empirical Lagrangian~(see Lemma~\ref{T:danskin}).

While effective, there are clear advantages in letting~$\nu$ be an optimization variable. In practice, not only does it lead to improved performance~(see Section~\ref{sect:experiments}), but it has the advantage of precluding the need to manually adjust another hyperparameter, which can be challenging and often requires domain-specific knowledge. Indeed, the value of~$\nu$ depends on the underlying learning task~(model, losses, dataset), making it difficult to transfer across applications and highly dependent on domain knowledge. What is more, if not done carefully, it can hinder generalization guarantees for the solution.

This issue is, in fact, at the core of the theoretical advantage of \eqref{P:empirical_dual}. Indeed, note that classical learning theory~\cite{vapnik2013nature, shalev2014understanding} provides generalization bounds only for the aggregated objective and not each individual penalty term, i.e., for the value of the Lagrangian rather than the adversarial and nominal losses in~\eqref{P:main}. In contrast, Proposition~\ref{T:dual} provides generalization guarantees both in terms of near-optimality and near-feasibility by leveraging the constrained learning theory developed in~\cite{chamon2020probably,chamon2020empirical}.

\newpage
\section{Proof of Proposition 3.1}\label{app:proof-prop-3.1}

Start by writing the primal problem~\eqref{P:semi-infinite} in Lagragian form~\cite[Ch. 4]{bertsekas2009convex}. Explicitly,
\begin{prob}\label{P:primal_lagrangian}
	P_R^\star = \min_{\btheta \in \Theta, t \in L^p}\ \max_{\bar{\lambda} \in L^q_+}\ L_\textup{\ref{P:semi-infinite}}(\btheta, t, \bar{\lambda})
		\text{,}
\end{prob}
where~$L^q_+$ denotes the subspace of almost everywhere non-negative functions of~$L^q$ for~$(1/p) + (1/q) = 1$.  Here the Lagrangian $L_\textup{\ref{P:semi-infinite}}(\btheta, t, \bar{\lambda})$ is defined as
\begin{equation}\label{E:lagrangian_sip}
\begin{aligned}
	L_\textup{\ref{P:semi-infinite}}(\btheta, t, \bar{\lambda}) &= \E_{(\bv x, y) \sim \calD} \!\big[ t(\bv x,y) \big]
		+ \int \bar{\lambda}(\bv x,\bdelta,y) \left[ \ell\big( f_{\btheta}(\xplusd), y \big) - t(\bv x,y) \right]
			d\bv x d\bdelta dy
	\\
	{}&= \int t(\bv x,y) \left[ \fkp(\bv x,y) - \int \bar{\lambda}(\bv x,\bdelta,y) d\bdelta \right] d\bv x dy + \int \bar{\lambda}(\bv x,\bdelta,y) \ell\big( f_{\btheta}(\xplusd), y \big) d\bv x d\bdelta dy
		\text{,}
\end{aligned}
\end{equation}
where we used the density~$\fkp$ of the data distribution~$\calD$. Then, notice that~\eqref{P:primal_lagrangian} can be written iteratively as
\begin{prob}\label{P:primal_lagrangian2}
	P_R^\star = \min_{\btheta \in \Theta}\ p(\btheta) \quad\text{where}\quad p(\btheta) = \min_{t \in L^p}\ \max_{\bar{\lambda} \in L^q_+}\ L_\textup{\ref{P:semi-infinite}}(\btheta, t, \bar{\lambda})
		\text{.}
\end{prob}
Observe that $\btheta$ is constant in the definition of $p(\btheta)$. Since~\eqref{E:lagrangian_sip} is a linear function of~$t$, $p(\btheta)$ is the optimal value of a linear program parametrized by~$\btheta$. Hence, strong duality holds~\cite[Ch. 4]{bertsekas2009convex} and we obtain that
\begin{equation}\label{E:primal_function2}
	p(\btheta) = \max_{\bar{\lambda} \in L^q_+}\ d_\textup{\ref{P:semi-infinite}}(\bar{\lambda})
		\quad\text{where}\quad d_\textup{\ref{P:semi-infinite}}(\bar{\lambda}) = \min_{t \in L^p}\ L_\textup{\ref{P:semi-infinite}}(\btheta, t, \bar{\lambda})
		\text{,}
\end{equation}
for the dual function~$d_\textup{\ref{P:semi-infinite}}$. Since~$t$ is unconstrained and~$L$ is linear in~$t$, the dual function either vanishes for~$\fkp(\bv x,y) = \int \bar{\lambda}(\bv x,\bdelta,y) d\bdelta$ or diverges to~$-\infty$. From~\eqref{E:lagrangian_sip} and~\eqref{E:primal_function2}, we thus obtain that
\begin{equation}\label{E:primal_function3}
\begin{aligned}
	p(\btheta) = \max_{\bar{\lambda} \in L^q_+}& &&\int \bar{\lambda}(\bv x,\bdelta,y) \ell\big( f_{\btheta}(\xplusd), y \big) d\bv x d\bdelta dy
	\\
	\subjectto& &&\int \bar{\lambda}(\bv x,\bdelta,y) d\bdelta = \fkp(\bv x,y)
\end{aligned}
\end{equation}
To conclude, notice that since~$\bar{\lambda}$ is almost everywhere non-negative, it must be that~$\bar{\lambda}(\bv x,\bdelta,y) = 0$ for all~$\bdelta \in \Delta$ whenever~$\fkp(\bv x, y) = 0$. The measure induced by~$\bar{\lambda}$ is therefore absolutely continuous with respect~$\calD$. We can therefore rewrite~\eqref{E:primal_function3} in terms of the Radon-Nykodim derivative,
\begin{align}
    \lambda(\bdelta \mid \bv x, y) = \bar{\lambda}(\bv x,\bdelta,y) / \fkp(\bv x, y)
\end{align}
which yields~\eqref{E:primal_function} as desired.

\newpage
\section{Proof of Proposition 3.2}

For the sake of completeness, before proving Proposition 3.2, in Section~\ref{app:prelims-prop-3.2} we provide a short discussion of the preliminary material needed to prove the proposition.  The majority of this exposition is adapted from Rockafellar and Wets' \emph{Variational Analysis}~\cite{rockafellar2009variational}.  Following this, in Section~\ref{app:prelim-lemma} we present a lemma which establishes the decomposability of $\calP^q$ over $\Omega$, which is crucial in proving the proposition.  Finally, in Section~\ref{app:final-proof-of-3.2} we provide the full proof of Proposition~3.2.

\subsection{Preliminaries}\label{app:prelims-prop-3.2}

Throughout these preliminaries, we let the tuple $(T, \calA)$ denote a measurable space, where $T$ is a nonempty set and $\calA$ is a $\sigma$-algebra of measurable sets belonging to $T$.  Furthermore, we let $\bar\R$ denote the extended real-line, and we will use $\mu$ to refer to an arbitrarily defined measure over the measurable space~$(T,\calA)$.\footnote{In some cases, we will also use $\mu$ to denote the Lebesgue measure on $\R^d$; this distinction will be made clear when we use this convention.}  By $\calG$ we denote an arbitrary space of measurable functions $g:T\to\R^n$.  To this end, given an integrand $f:T\times\R^n\to\bar\R$, we will consider integral functionals of the form
\begin{align}
    I_f[g] = \int_T f(t, g(t)) \mu(dt)
\end{align}
\noindent To begin our preliminaries, we first recall the definition of a normal integrand.

\begin{definition}[Normal integrand]
A function $f:T\times\R^n\to\bar\R$ is called a \textbf{normal integrand} if its epigraphical mapping $S_f:T\to\R^n\times\R$ defined by
\begin{align}
    S_f(t) \triangleq \text{epi} f(t, \cdot) = \left\{ (x,\alpha) \in\R^n\times\R \big| f(t,x)\leq \alpha\right\} \label{eq:epi-def}
\end{align}
is closed valued and measurable.
\end{definition}

\begin{definition}[Carath\'eodory integrand]
A function $f:T\times\R^n\to\R$ is called a \textbf{Carath\'eodory integrand} if it is measurable in $t$ for each $x$ and continuous in $x$ for each $t$.
\end{definition}

\noindent Notably, Carath\'eodory integrands are a special case of normal integrands (see e.g., \cite[Ex.\ 14.29]{rockafellar2009variational}).  Next, recall the definition of a \emph{decomposable space}.  

\begin{definition}[Decomposable space]~\label{def:decomposable}
A space $\calG$ of measurable functions $g:T\to\R^n$ is \textbf{decomposable} in association with a measure $\mu$ on $\calA$ if for every function $g_0\in\calG$, for every set $A\in\calA$ with $\mu(A)<\infty$, and for every bounded, measurable function $g_1:A\to\R^n$, the space $\calG$ contains the function $g:T\to\R^n$ defined by
\begin{align}
    g(t) = \begin{cases}
        g_0(t) &\quad\text{for } t\in T\backslash A, \\
        g_1(t) &\quad\text{for } t\in A.
    \end{cases}
\end{align}
\end{definition}

\noindent Note that the space $\calM(T, A)$ of measurable functions $g:T\to\R^n$ is decomposable, as are the Lebesgue spaces $L^p(T, \calA, \mu)$ for all $p\in[1,\infty]$ (see e.g.,~\cite[Ch. 14]{rockafellar2009variational}).  As we will see, the decomposability of the Lebesgue spaces is integral to the proof of Proposition 3.2.  However, before proceeding to the proof, we first restate a crucial result concerning the interchangability of minimization and integration, which relies on this notion of decomposability defined above.

\begin{theorem}[Thm.\ 14.60 in~\cite{rockafellar2009variational}] \label{thm:interchange}
Let $\calG$ be a space of measurable functions from $T$ to $R^n$ that is decomposable relative to a $\sigma$-finite measure $\mu$ defined on $\calA$.  Let $f:T\times\R^n$ be a normal integrand.  Then the minimization of $I_f$ over $\calG$ can be reduced to a pointwise minimization in the sense that, as long as $I_f \not\equiv 0$ on $\calG$, one has
\begin{align}
    \inf_{g\in\calG} \int_T f(t, g(t))\mu(dt) = \int_T \left[ \inf_{x\in\R^n} f(t,x) \right]\mu(dt)
\end{align}
Moreover, as long as this common value is not $-\infty$, one has for $\bar g\in\calG$ that
\begin{align}
    \bar g\in\argmin_{g\in\calG} I_f[g] \iff \bar g(t) \in\argmin_{x\in\R^n} f(t,x) \quad\text{for } \mu\text{-almost every } t\in T.
\end{align}
\end{theorem}
\noindent The utility of this result is that under the assumptions that function class $\calG$ is decomposable, the integrand $f$ is normal, and $I_f$ is finite over $\calG$, it holds that the minimization and integration operations can be exchanged.  Furthermore, note that this result is more general than we need; indeed, all of the integrands we work with in the next subsection are Carath\'eodory and hence normal.

\subsection{A preliminary lemma}\label{app:prelim-lemma}

The first step toward proving Proposition~\ref{T:lambda_star} is to show that the space $\calP_q$ is decomposable over the data space $\Omega$.  We state this result in the following lemma, as it may be of expository interest as a warm-up before the proof of Proposition 3.2.
\begin{lemma}[Decomposability of $\calP^q$ over $\Omega$]  \label{lem:decomp}
The space $\calP^q$ of distributions of $\Delta$ defined in Proposition~\ref{T:lambda_star} is decomposable over $\Omega = \calX\times\calY$ in the sense of definition~\ref{def:decomposable}.
\end{lemma}

\begin{proof}   
Let $\mu$ denote the Lebesgue measure on $\R^d$. Recall that $\calP^q$ is the subset of $L^p$ containing functions $\lambda$ with the following properties:
\begin{enumerate}[leftmargin=2cm]
    \item[(P1)] $\lambda(\cdot|\bv x,y)$ is almost everywhere non-negative on $\Omega$,
    \item[(P2)] $\lambda(\cdot | \bv x,y)$ is absolutely continuous with respect to $\fkp(\bv x,y)$,
    \item[(P3)] $\int_\Delta \lambda(\bdelta | \bv x,y) \mu(d\bdelta)=1 \quad\text{for }\fkp\text{-almost every } (\bv x, y)\in\Omega$.
\end{enumerate}
To show that $\calP^q$ is decomposable over $\Omega$, first let $\lambda,\lambda'\in\calP^q$ and $A\subseteq\Omega$ with $\mu(A)<\infty$ be arbitrarily chosen.  Define the functional
\begin{align}
    \bar \lambda(\bdelta | \bv x, y) = \begin{cases}
        \lambda(\bdelta | \bv x, y) &\quad\text{for } (\bv x, y)\in \Omega\backslash A, \\
        \lambda'(\bdelta | \bv x, y) &\quad\text{for } (\bv x, y) \in A.
    \end{cases}
\end{align}
Our goal is to show that $\bar\lambda$ is an element of $\calP^q$.  To begin, observe that by (P1), $\lambda$ and $\lambda'$ are almost everywhere non-negative, and therefore so is $\bar\lambda$.  Further, by (P2), both $\lambda$ and $\lambda'$ are absolutely continuous with respect to $\fkp$.  Thus, observe that if $B\in\calB$ such that $\fkp(B) = 0$, then $\lambda(\delta|
B) = \lambda'(\delta|B) = 0$, and thus it holds that $\bar\lambda(\delta|B) = 0$, proving that $\bar\lambda\ll \fkp$.  Finally, note that by (P3), both $\lambda$ and $\lambda'$ are normalized along $\Delta$.  Thus, for any fixed $(\bv x,y)\in\Sigma$, it holds that $\int_\Delta\bar\lambda(\bdelta|\bv x,y)\mu(d\delta) = 1$.  Thus, it holds that $\bar\lambda\in\calP^q$, as was to be shown.
\end{proof}

\subsection{Proof of Proposition 3.2}\label{app:final-proof-of-3.2}

Ultimately, there are three main steps to this proof.  (1) First, we argue that Thm.~\ref{thm:interchange} applies to the maximization over $\lambda$, so that the expectation over the data distribution $\calD$ and the maximization over $\lambda\in\calP^q$ can be interchanged.  (2) We argue that strong duality holds for the inner problem induced by pushing the maximization inside the expectation.  (3) We find a closed-form solution for the dual problem, proving the claim of the proposition.

\begin{proof}
\textbf{Step 1.} To begin, we argue that the maximization over $\lambda$ and the expectation over the data distribution $\calD$ can be interchanged.  To do so, we define the function $F:\Omega\times\calP^2\to\R$
\begin{align}
    F\big((\bv x,y), \lambda\big) \triangleq \E_{\bdelta\sim\lambda(\bdelta|\bv x,y)} \big[ \ell(f_{\btheta}(\bv x+\bdelta),y)\big]
\end{align}
so that the optimization problem in~\eqref{P:primal_sip} can be written as
\begin{align}
    P_R^\star = \min_{\btheta\in\Theta} \: p(\btheta) \quad\text{where}\quad p(\btheta)\triangleq \max_{\lambda\in\calP^2} \: \E_{(\bv x,y)\sim\calD} \big[ F\big((\bv x,y),\lambda\big)\big].
\end{align}
Now observe that by construction $F$ is measurable in $(\bv x,y)$ for each $\lambda$, and continuous (in fact, linear) in $\lambda$ for each~$(\bv x,y)$.  Thus, $F$ is a Carath\'eodory integrand, and as $\calP^2$ is decomposable over $\Omega$ by Lemma~\ref{lem:decomp},  Thm.~\ref{thm:interchange} applies to $p(\btheta)$.  Therefore, the maximization in the definition of $p(\btheta)$ can be pushed inside the expectation over $(\bv x,y)\sim\calD$, yielding the following equality:
\begin{align}
    p(\btheta) = \max_{\lambda\in\calP^2} \: \E_{(\bv x,y)\sim\calD} \big[ F\big((\bv x,y),\lambda\big)\big] = \E_{(\bv x,y)\sim\calD} \left[ \max_{\lambda\in\calP^2} F((\bv x,y),\lambda)\right]. \label{eq:interchange}
\end{align}

\textbf{Step 2.} Next, we argue that the inner maximization problem in the expression on the RHS of~\eqref{eq:interchange} is strongly dual.  To begin, notice that this inner maximization is now performed separately for each data point $(\bv x,y)\in\Omega$.  We therefore proceed by considering the solution of the inner problem for an arbitrary but fixed data point~$(\bv{\bar x},\bar y)$.  To this end, first let $\lambda^\star$ be the solution to the inner problems for this fixed pair $(\bv{\bar x},\bar y)$, i.e. $\lambda^\star$ achieves the optimal value in
\begin{align}
    u(\btheta) = u(\btheta, \bv{\bar x}, \bar{y}) = \max_{\lambda\in\calP^2} F((\bv{\bar x},\bar y),\lambda) \label{eq:inner-prob}
\end{align}
Now let $\fkm(\Delta)$ denote the Lebesgue measure of $\Delta$, and consider that H\"older's inequality implies that $\norm{\lambda^\star}_{L^1} \leq \fkm(\Delta)^{1/2} \norm{\lambda^\star}_{L^2}$.  Further, as each feasible $\lambda\in\calP^2$ is normalized over $\Delta$, it holds that
\begin{align}
    \frac{1}{\fkm(\Delta)} \leq \norm{\lambda^\star}_{L^2}^2. \label{eq:holder-normalization}
\end{align}
Thus, since $\calP^2\subset L_+^2$, it holds that $\lambda^\star\in L_+^2$, and thus there exists a constant $c$ satisfying $1/\fkm(\Delta)\leq c<\infty$ such that
\begin{align}
    \norm{\lambda^\star}_{L^2}^2 = \int_\Delta F((\bv{\bar x},\bar y), \delta)^2 d\delta \leq c.
\end{align}
Accordingly, we can rewrite~\eqref{eq:inner-prob} in an equivalent way as follows:
\begin{prob} \label{P:primal_function_mod}
    u(\btheta) = &\max_{\lambda\in L_+^2} &&\E_{\bdelta\sim\lambda(\bdelta|\bv x,y)} \big[ \ell(f_{\btheta}(\bv {\bar x}+\bdelta),\bar y)\big] \\
    &\subjto &&\int_\Delta \lambda(\bdelta|\bv {\bar x},\bar y) d\delta = 1, \quad \int_\Delta \lambda(\bdelta|\bv {\bar x},\bar y)^2 d\delta \leq c.
\end{prob}
Notice that~\eqref{P:primal_function_mod} is a convex quadratic program in the optimization variable $\lambda$.  Furthermore, note that if $c = 1/\fkm(\Delta)$ (i.e., equality is achieved in the expression~\eqref{eq:holder-normalization} derived from H\"older's inequality), then the feasible set is a singleton which is equivalent in $L_2$ to $\lambda(\bdelta | \bv {\bar x},\bar y) = 1/\fkm(\Delta)$.  Alternatively, if $c > 1/\fkm(\Delta)$, then $\lambda(\bdelta|\bv {\bar x},\bar y)$ is a strictly feasible point, and thus Slater's condition holds.  In either case, we find that~\eqref{P:primal_function_mod} is strongly dual~\cite[Ch. 4]{bertsekas2009convex} and so we can write
\begin{equation}\label{E:primal_function_dual}
	u(\btheta) = \min_{\gamma \geq 0 \text{, } \mu \in \setR}\ d_\textup{\ref{P:primal_function_mod}}(\btheta, \gamma, \mu)
    	\text{,}
\end{equation}
for the dual function
\begin{equation*}
	d_\textup{\ref{P:primal_function_mod}}(\btheta, \gamma, \mu) = \max_{\lambda \in L^2_+}\ \int_\Delta \big[
		 \ell(f_{\btheta}(\bv {\bar x} + \bdelta)\lambda(\bdelta|\bv {\bar x},\bar y)
		- \gamma \lambda(\bdelta | \bv{\bar x}, \bar y)^2
		- \mu \lambda(\bdelta| \bv{\bar x}, \bar y)
	\big] d\bdelta + \gamma c + \mu
		\text{.}
\end{equation*}

\textbf{Step 3.}  Finally, we find a closed-form expression for the solution to the dual problem derived above.  To do so, an entirely similar argument to the one given in Lemma~\ref{lem:decomp} shows that $L_+^2$ is decomposable.  And indeed, as the integrand in the above primal function is clearly Carath\'eodory, we can again apply Theorem~\ref{thm:interchange} to $d_\textup{\ref{P:primal_function_mod}}(\btheta, \gamma, \mu)$:
\begin{align}
    d_\textup{\ref{P:primal_function_mod}}(\btheta, \gamma, \mu) = \int_\Delta  \left\{ \max_{\lambda\in L_+^2} \: \ell(f_{\btheta}(\bv {\bar x} + \bdelta)\lambda(\bdelta|\bv {\bar x},\bar y)
		- \gamma \lambda(\bdelta | \bv{\bar x}, \bar y)^2
		- \mu \lambda(\bdelta| \bv{\bar x}, \bar y) \right\} d\delta + \gamma c + \mu.
\end{align}
A straightforward calculation of the inner maximization problem shown above yields
\begin{equation}\label{E:optimal_lambda}
	\lambda^\star(\bdelta | \bv {\bar x}, \bar y) = \left[ \frac{\ell(f_{\btheta}(\bv{\bar x} + \bdelta), \bar y) - \mu}{2\gamma} \right]_+
		\text{,}
\end{equation}
where~$[z]_+ = \max(0,z)$ denotes the projection onto the non-negative orthant.  From~\eqref{P:primal_function_mod}, $\mu$ is chosen so as to meet the normalization constraint, i.e., so that
\begin{equation*}
	\int_\Delta \left[ \frac{\ell(f_{\btheta}(\bv{\bar x} + \bdelta), \bar y) - \mu}{2\gamma} \right]_+ d\bdelta = 1 \iff \int_\Delta \left[ \ell(f_{\btheta}(\mathbf{\bar x}+\bm\delta), \bar y) - \mu\right]_+ d\delta = 2\gamma
		\text{.}
\end{equation*}
To conclude, notice that due to the strong duality of~\eqref{P:primal_function_mod}, for each value of~$c < \infty$ there is a value of~$\gamma > 0$ such that~\eqref{E:optimal_lambda} is a solution of~\eqref{P:primal_function_mod}~\cite[Ch. 4]{bertsekas2009convex}. Also, since~$(\bv {\bar x}, \bar y)$ were chosen arbitrarily, \eqref{E:optimal_lambda} holds for all data point. Given that the space is decomposable, these solutions can be pieced together, yielding the desired result.
\end{proof}
\newpage
\section{Proof of Proposition~\ref{T:dual}}

We proceed here as in~\cite{chamon2020probably}. However, we deviate slightly from the proof of the parametrization gap~\cite[Prop.~2 in Appendix~B.1]{chamon2020probably} to account for the maximization in the robust loss.  In particular, the proof of Proposition 3.6 is organized in the following way:
\begin{enumerate}
    \item First, in Section~\ref{app:param-gap} we bound the deviation between the primal problem~\eqref{P:main} and dual problem~\eqref{P:dual_main}.  The result is summarized in Lemma~\ref{T:param_gap}.
    \item Next, in Section~\ref{app:prelims-3.6}, we review two results needed to complete the proof of Proposition~3.6 concerning the continuity and differentiability of the dual objective.
    \item Finally, in Section~\ref{S:empirical_gap} we leverage the preliminaries provided in Section~\ref{app:prelims-3.6} to complete the proof of the proposition.  This result is summarized in Proposition~\ref{T:empirical}.
\end{enumerate}

\noindent Ultimately, the result in Proposition~\ref{T:dual} is obtained by combining the results in Lemma~\ref{T:param_gap} and Proposition~\ref{T:empirical} and using the union bound.

\subsection{Bounding the parametrization gap} \label{app:param-gap}

In this section, we are interested in the relationship between the statistical problem~\eqref{P:main} and its dual problem. In particular, the dual problem to~\eqref{P:main} can be written in the following way
\begin{prob}[\textup{DI}]\label{P:dual_main}
	D^\star \triangleq \max_{\nu \geq 0}\ \min_{\btheta \in \Theta}\ L(\btheta,\nu)
\end{prob}
for the Lagrangian
\begin{equation}\label{E:lagrangian_main}
	L(\btheta,\nu) = \E \Big[ \max_{\dinD}\ \ell\big( f_{\btheta}(\xplusd), y \big) \Big]
		+ \nu \Big[ \E \left[ \ell\big(f_{\btheta}(\bv x), y\big) \right] - \epsilon \Big]
			\text{.}
\end{equation}
The goal of this subsection is to prove the following lemma, which establishes bounds on the error induced by the parameterization space $\Theta$.

\begin{lemma}\label{T:param_gap}
Under the conditions of Prop.~\ref{T:dual}, the value~$D^\star$ of~\eqref{P:dual_main} is related to the value~$P^\star$ of~\eqref{P:main} by
\begin{equation}\label{E:param_gap}
	P^\star - M \alpha \leq D^\star \leq P^\star
		\text{.}
\end{equation}
\end{lemma}

\begin{proof}

The result in Lemma~\ref{T:param_gap} is trivial when the hypothesis class~$\calH$ induced by the parametrization is convex. In this case, \eqref{P:main} is a convex program and Assumption~\ref{A:slater}~(Slater's condition) implies that it is strongly dual~\cite[Ch. 4]{bertsekas2009convex}. In other words, $P^\star = D^\star$. Hence, we are interested in the setting in which~$\calH$ is not convex, but is still a rich parametrization as per~\eqref{A:parametrization} such as the class of CNNs.

In the nonconvex case, the upper bound in~\eqref{E:param_gap} is a simple consequence of weak duality~\cite[Ch. 4]{bertsekas2009convex}. To obtain the lower bound, consider the variational problem
\begin{prob}\label{P:variational_pert}
	\tilde{P}^\star \triangleq &\minimize_{\phi \in \bar{\calH}} &&\E_{(\bv x, y) \sim \calD} \left[
	    	\max_{\dinD}\ \ell\big( \phi(\xplusd), y \big)
	    \right]
    \\
	&\subjto &&\E_{(\bv x,y) \sim \calD} \left[ \ell\big(\phi(\bv x), y\big) \right] \leq \epsilon - M \alpha
\end{prob}
for~$\bar{\calH} = \conv(\calH)$. Let~$\tilde{\phi}^\star \in \bar{\calH}$ be a solution of~\eqref{P:variational_pert} associated with~$\delta^\star(\bv x, y)$, the perturbations that attains the maximum in its objective. Since~$\Delta$ is a compact set by assumption, there indeed exists a perturbation~$\delta^\star \in \Delta$ that achieves the maximum. Since~$\bar{\calH}$ is convex, \eqref{P:variational_pert} is now a convex optimization problem~(recall that the pointwise maximum of convex functions is convex~\cite[Prop 1.1.6]{bertsekas2009convex}) which therefore has a strictly feasible point~$f_{\btheta^\prime} \in \calH \subset \bar{\calH}$~(Assumption~\ref{A:slater}). Hence, it is strongly dual~\cite[Ch.~4]{bertsekas2009convex} and
\begin{equation}\label{E:dual_perturbed}
	\tilde{P}^\star = \max_{\tilde{\nu} \geq 0}\ \min_{\phi \in \bar{\calH}}\ \tilde{L}(\phi,\tilde{\nu})
		= \tilde{L}(\tilde{\phi}^\star,\tilde{\nu}^\star)
		\text{,}
\end{equation}
where~$\tilde{\nu}^\star$ achieves the maximum in~\eqref{E:dual_perturbed} for the Lagrangian\footnote{For clarity, we omit the distribution~$\calD$ over which the expectations are taken.}
\begin{equation}\label{E:lagrangian_perturbed}
	\tilde{L}(\phi,\tilde{\nu}) = \E \Big[ \max_{\dinD}\ \ell\big( \phi(\xplusd), y \big) \Big]
		+ \tilde\nu \Big[ \E \left[ \ell\big(\phi(\bv x), y\big) \right] - \epsilon + M \alpha \Big]
\end{equation}
To proceed, notice from~\eqref{P:dual_main} that
\begin{equation*}
	D^\star \geq \min_{\btheta \in \Theta}\ L( \btheta, \nu )
		\text{,} \quad \text{for all } \nu \geq 0
		\text{,}
\end{equation*}
and that since~$\calH \subseteq \bar{\calH} = \conv(\calH)$, we obtain
\begin{equation}\label{E:lower_bound1}
	D^\star	\geq \min_{\btheta \in \Theta} L( \btheta, \tilde{\nu}^\star)
		\geq \min_{\phi \in \bar{\calH}} \tilde{L}( \phi, \tilde{\nu}^\star )
		\text{.}
\end{equation}
Using the strong duality of~\eqref{P:variational_pert}, the expression written above in~\eqref{E:lower_bound1} yields
\begin{equation}\label{E:lower_bound2}
	D^\star \geq \min_{\phi \in \bar{\calH}} \tilde{L}( \phi, \tilde{\nu}^\star ) = \tilde{P}^\star
		= \E \!\Big[ \ell\big( \phi^\star(\bv x + \delta^\star(\bv x, y)), y \big) \Big]
		\text{.}
\end{equation}
Now note that to obtain the lower bound in~\eqref{E:param_gap}, it suffices to show that
\begin{align}
    \E \!\Big[ \ell\big( \phi^\star(\bv x + \delta^\star(\bv x, y)), y \big) \Big] \geq P^\star - M\alpha
\end{align}
To obtain this lower bound, notice from Assumption~\ref{A:parametrization} that there exists~$\tilde{\btheta}^\dagger \in \Theta$ such that
\begin{align}
    \sup_{\bv x} \abs{\tilde{\phi}^\star(\bv x) - f_{\tilde{\btheta}^\dagger}(\bv x)} \leq \alpha.
\end{align}
For these parameters, it holds that
\begin{multline*}
	\abs{\E \Big[ \ell\big( \phi(\bv x + \delta^\star(\bv x, y)), y \big) \Big]
		- \E \Big[ \max_{\dinD}\ \ell\big( f_{\tilde{\btheta}^\star}(\xplusd), y \big) \Big]}
	\leq
	\\
	\E \left[ \abs{\ell\big( \phi(\bv x + \delta^\star(\bv x, y)), y \big)
		- \max_{\dinD}\ \ell\big( f_{\tilde{\btheta}^\star}(\xplusd), y \big)} \right]
	\leq
	\\
	\E \left[ \abs{\ell\big( \phi(\bv x + \delta^\star(\bv x, y)), y \big)
		- \ell\big( f_{\tilde{\btheta}^\star}(\bv x + \delta^\star(\bv x, y)), y \big)} \right]
		\text{,}
\end{multline*}
where the first inequality is due to the convexity of the absolute value~(Jensen's inequality) and the second inequality follows from the fact that~$\delta^\star$ is a suboptimal solution of~$\max_{\dinD}\ \ell\big( f_{\tilde{\btheta}^\star}(\xplusd), y \big)$. Using the Lipschitz continuity of the loss and Assumption~\ref{A:parametrization}, we conclude that
\begin{equation}\label{E:bound_lipschitz1}
	\abs{\tilde{P}^\star - \E \Big[ \max_{\dinD}\ \ell\big( f_{\tilde{\btheta}^\star}(\xplusd), y \big) \Big]}
		\leq M \E \left[ \abs{\phi(\bv x + \delta^\star(\bv x, y))
			- f_{\tilde{\btheta}^\star}(\bv x + \delta^\star(\bv x, y))} \right]
		\leq M \alpha
		\text{.}
\end{equation}
Using a similar argument, we also obtain that
\begin{equation}\label{E:bound_lipschitz2}
	\abs{\E \Big[ \ell\big( \tilde{\phi}^\star(\bv x), y \big) \Big]
		- \E \Big[ \ell\big( f_{\tilde{\btheta}^\star}(\bv x), y \big) \Big]}
		\leq M \alpha
		\text{.}
\end{equation}
Hence, given that~$\tilde{\phi}^\star(\bx)$ is feasible for~\eqref{P:variational_pert}, \eqref{E:bound_lipschitz2} implies that~$\tilde{\btheta}^\star$ is feasible for~\eqref{P:main}. By optimality, $P^\star \leq \E \left[ \max_{\dinD}\ \ell\big( f_{\tilde{\btheta}^\star}(\xplusd), y \big) \right]$. Now recalling~\eqref{E:lower_bound2}, we conclude that
\begin{align*}
	D^\star	&\geq \E \!\Big[ \ell\big( \phi^\star(\bv x + \delta^\star(\bv x, y)), y \big) \Big]
	\\
	{}&\geq P^\star + \E \!\Big[ \ell\big( \phi^\star(\bv x + \delta^\star(\bv x, y)), y \big) - \max_{\dinD}\ \ell\big( f_{\tilde{\btheta}^\star}(\xplusd), y \big) \Big]
	\\
	{}&\geq P^\star - M \alpha
		\text{,}
\end{align*}
where the last inequality follows from~\eqref{E:bound_lipschitz1}.
\end{proof}

\subsection{Preliminaries for proving the empirical gap} \label{app:prelims-3.6}

Before proceeding to considering the empirical gap of the dual problem, we state two preliminary results which will be useful in proving the empirical gap in Section~\ref{S:empirical_gap}.  To this end, we first state the classical result known as Danskin's theorem.

\begin{theorem}[Danskin's Theorem]\label{T:danskin}
Consider the function 
\begin{align}
    F(w) = \max_{z\in\calZ}\ f(w,z) \label{eq:max-fn}
\end{align}
where $f:\R^n\times\calZ\to\bar \R$ and assume that the following three conditions hold:
\begin{enumerate}
    \item[(i)] $f(\cdot, z)$ is convex in $w$ for each $z\in\calZ$;
    \item[(ii)] $f(w, \cdot)$ is continuous in $z$ for each $w$ in a certain neighborhood of a point $w_0$;
    \item[(iii)] The set $\calZ$ is compact.
\end{enumerate}
Then it holds that
\begin{equation}\label{E:danskin}
	\partial F(x_0) = \conv\left(
		\bigcup_{z\in\hat\calZ(w_0)} \partial_w f(w_0, z) \big)
	\right)
\end{equation}
where $\hat\calZ(w)$ denotes the set of $z\in\calZ$ at which $F(w) = f(w,z)$.
\end{theorem}
\noindent The interested reader can find a full proof of Danskin's theorem in~\cite[Thm.~2.87]{ruszczynski2011nonlinear}.  

Next, we note that in the proof presented in Section~\ref{S:empirical_gap}, it will be necessary to verify that a function analogous to $F(w)$ defined in~\eqref{eq:max-fn} which is defined as the pointwise maximum of continuous function $f(w,z)$ is continuous.  To verify this continuity, we will rely on the following result:
\begin{lemma} \label{lem:continuity-of-inf}
Fix any point $u_0\in\Phi$ and let $g:\Phi\times\R^n\to\bar \R$.  Denote 
\begin{align}
    v(u) = \inf_{w\in\Phi(u)} g(w, u) \quad\text{and}\quad \calS(u) = \argmin_{w\in\Phi(u)} g(w, u)
\end{align}
Now suppose that
\begin{enumerate}
    \item[(a)] The function $g(w,u)$ is continuous on $\Phi\times\R^n$;
    \item[(b)] The feasible set $\Phi$ is closed;
    \item[(c)] There exists a constant $\alpha\in\R$ and a compact set $C\subset\R^n$ such that for every $u$ in a neighborhood of $u_0$, the level set $\{w\in\Phi(u) : g(w,u)\leq \alpha\}$ is nonempty and contained in $C$;
    \item[(d)] For any neighborhood $\calN$ of $\calS(u_0)$, there exists a neighborhood $\calN_U$ of $u_0$ such that $\calN \cap \Phi(u) \neq \varnothing$ $\forall u\in\calN_U$.
\end{enumerate}
Then it holds that $v(u)$ is continuous at $u = u_0$.
\end{lemma}

\noindent Further details regarding this result as well as a full proof can be found in~\cite[Prop. 4.4]{bonnans2013perturbation}.

\subsection{Bounding the empirical gap}
\label{S:empirical_gap}

We now proceed by evaluating the empirical gap between the statistical dual problem~\eqref{P:dual_main} and its empirical version~\eqref{P:empirical_dual}.  In particular, our goal is to prove the following result.

\begin{proposition}\label{T:empirical}
Let~$\hat{\nu}$ be a solution of~\eqref{P:empirical_dual} with a finite~$D^\star$. Under the conditions of Theorem~\ref{T:dual}, there exists~$\bhtheta \in \argmin_{\btheta \in \Theta}\ \hat{L}(\btheta,\hat{\nu}^\star)$ such that
\begin{align}
	\big\vert D^\star - \hat{D}^\star \big\vert \leq (1 + \bar{\nu}) \max(\zeta_R(N), \zeta_N(N))
	\quad \text{(near-optimality)}
		\label{E:empirical_gap}
	\\
	\E_{(\bx,y) \sim \calD} \!\Big[ \ell_i\big( f_{\bm{{\hat{\theta}}}^\star}(\bx), y \big) \Big]
		\leq c_i + \zeta_i(N_i)
		\quad\text{(near-feasibility).}
		\label{E:empirical_feas}
\end{align}
hold with probability~$1-5\delta$, where~$\bar{\nu} = \max(\hat{\nu}^\star,\nu^\star)$, for~$\hat{\nu}^\star$ a solution of~\eqref{P:empirical_dual} and~$\nu^\star$ a solution of~\eqref{P:dual_main}, and~$D^\star$ and~$\hat{D}^\star$ as in~\eqref{P:dual_main} and~\eqref{P:empirical_dual} respectively.
\end{proposition}

\begin{proof}\textbf{(Near-optimality).}
Let~$\nu^\star$ and~$\hat{\nu}^\star$ be solutions of~\eqref{P:dual_main} and~\eqref{P:empirical_dual} respectively and consider the set of dual minimizers
\begin{equation*}
	\Theta^\dagger(\nu) = \argmin_{\theta \in \Theta} L(\btheta,\nu)
	\quad\text{and}\quad
	\hat{\Theta}^\dagger(\hat{\nu}) = \argmin_{\theta \in \Theta} \hat{L}(\btheta,\hat{\nu})
\end{equation*}
Using the optimality of~$\nu^\star$, it holds that
\begin{align*}
	D^\star - \hat{D}^\star &= \min_{\btheta \in \Theta} L(\btheta,\nu^\star)
		- \min_{\btheta \in \Theta} \hat{L}(\btheta,\hat{\nu}^\star)
	\\
	{}&\leq \min_{\btheta \in \Theta} L(\btheta,\nu^\star) - \min_{\btheta \in \Theta} \hat{L}(\btheta,\nu^\star)
		\text{.}
\end{align*}
Since~$\bhdtheta \in \hat{\Theta}^\dagger(\nu^\star)$ is suboptimal for~$L(\btheta,\nu^\star)$, we get
\begin{equation}\label{E:empirical_upper_bound}
	D^\star - \hat{D}^\star \leq L(\bhdtheta,\nu^\star) - \hat{L}(\bhdtheta,\nu^\star)
		\text{.}
\end{equation}
Using a similar argument yields
\begin{equation}\label{E:empirical_lower_bound}
	D^\star - \hat{D}^\star \geq L(\btheta^\dagger,\hat{\nu}^\star) - \hat{L}(\btheta^\dagger,\hat{\nu}^\star)
\end{equation}
for~$\btheta^\dagger \in \Theta^\dagger(\hat{\nu}^\star)$. Thus, we obtain that
\begin{equation}\label{E:gap_bound}
	\abs{D^\star - \hat{D}^\star} \leq
	\max \bigg\{
		\abs{L(\bhdtheta,\nu^\star) - \hat{L}(\bhdtheta,\nu^\star)},
		\abs{L(\btheta^\dagger,\hat{\nu}^\star) - \hat{L}(\btheta^\dagger,\hat{\nu}^\star)}
	\bigg\}
\end{equation}
Using the empirical bound from Assumption~\ref{A:empirical}, we obtain that
\begin{align}\label{E:vc_F_bound}
	\abs{L(\btheta,\nu) - \hat{L}(\btheta,\nu)} &\leq \zeta_R(N) + \nu \zeta_N(N)
		\text{,}
\end{align}
holds uniformly over~$\btheta$ with probability~$1-4\delta$.

\vspace{0.5\baselineskip}\noindent
\textbf{(Near-feasibility).}
The proof relies on characterizing the superdifferential of the dual function
\begin{align}
    \hat{d}(\nu) = \min_{\btheta \in \Theta}\ \hat{L}(\btheta,\nu)
\end{align}
from~\eqref{P:empirical_dual}. Explicitly, we say~$p \in \setR$ is a \emph{supergradient} of~$\hat{d}$ at~$\nu$ if
\begin{equation}\label{E:supergrad}
	\hat{d}(\nu^\prime) \leq \hat{d}(\nu) + p (\nu^\prime - \nu)
		\text{, for all } \nu^\prime \geq 0
		\text{.}
\end{equation}
The set of all supergradients of~$\hat{d}$ at~$\nu$ is called the \emph{superdifferential} of~$\hat{d}$ at~$\nu$ and is denoted~$\partial \hat{d}(\nu)$. To characterize the supperdifferential, first let~$\Theta^\dagger(\nu) \in \argmin_{\btheta \in \Theta} \hat{L}(\btheta, \nu)$ for the Lagrangian~$\hat{L}$ and define the constraint slack as
\begin{equation}\label{E:slack_vector}
	s(\btheta) = \left[ \frac{1}{N} \sum_{n=1}^N \ell\big( f_{\btheta}(\bv x_n), y_n \big) - \epsilon \right]_+
		\text{.}
\end{equation}
Now by rewriting $\hat{d}(\nu)$ as follows
\begin{align}
    \hat{d}(\nu) = \min_{\btheta \in \Theta} \hat{L}(\btheta,\nu) = -\max_{\btheta \in \Theta} -\hat{L}(\btheta,\nu),
\end{align}
we argue that the conditions of Thm.~\ref{T:danskin} are satisfied.  To begin, we note that because~$\hat{L}(\btheta, \cdot)$ is affine in $\nu$ for all~$\btheta \in \Theta$ and~$\Theta$ is compact, we immediately meet conditions~(i) and~(iii) of Thm.~\ref{T:danskin}.  Thus, it suffices to show that~$\hat{L}(\cdot,\nu)$ is continuous in $\btheta$ for all~$\nu \geq 0$.  

To prove the continuity of $\hat{d}(\nu)$, we will seek to show that the conditions of Lemma~\ref{lem:continuity-of-inf} are satisfied for $\hat d(\nu)$.  In this way, first recall that~$\ell(\cdot,y)$ is continuous for all~$y \in \calY$ and~$f_{\btheta}(\bv x)$ is differentiable with respect to~$\btheta$ and~$\bv x$.  Therefore, it holds that~$\ell(f_{\btheta}(\bv x),y)$ is continuous on~$\Omega \times \Theta$.  Thus, property (a) in Lemma~\ref{lem:continuity-of-inf} holds.  Furthermore, the fact that~$\Delta$ is compact and fixed with respect to~$\btheta$ establishes~(b) and~(d).  Finally, condition~(c) holds by observing that the perturbation~$\bdelta$ are finite dimensional and that~$\ell(f_{\btheta}(\xplusd),y)$ is uniformly bounded.

Having established the continuity of $\hat{d}(\nu)$, altogether we have shown that the conditions~(i)--(iii) of Thm.~\ref{T:danskin} are satisfied.  Thus, using the fact that~$\hat{d}(\nu) = \min_{\btheta \in \Theta} \hat{L}(\btheta,\nu) = -\max_{\btheta \in \Theta} -\hat{L}(\btheta,\nu)$, Thm.~\ref{T:danskin} yields the following result:
\begin{equation}\label{E:danskin-our-setting}
	\partial d(\bmu) = \conv\left(
		\bigcup_{\btheta^\dagger \in \Theta^\dagger(\bmu)} \bs\big( \btheta^\dagger \big)
	\right)
		\text{.}
\end{equation}
To complete the proof, we assume toward contradiction that for all~$\hat{\btheta}^\dagger \in \hat{\Theta}^\dagger(\hat{\nu}^\star)$ it holds that
\begin{equation*}
	\frac{1}{N} \sum_{n = 1}^{N} \ell\big( f_{\hat{\btheta}^\dagger}(\bv x_{n}), y_{n} \big) > \epsilon
		\text{.}
\end{equation*}
Then, from our discussion of the superdifferential, $\bzero \notin \partial d(\hat{\nu}^\star)$, which contradicts the optimality of~$\hat{\nu}^\star$. Hence, there must be~$\hat{\btheta}^\dagger \in \hat{\Theta}^\dagger(\bhmu)$ such that~$\frac{1}{N} \sum_{n = 1}^{N} \ell\big( f_{\hat{\btheta}^\dagger}(\bv x_{n}), y_{n} \big) \leq \epsilon$. For those parameters, the uniform bound in Assumption~\ref{A:empirical}, yields that, with probability~$1-\delta$ over the data,
\begin{equation}\label{E:constraint_bound}
	\E_{(\bx,y) \sim \calD} \!\Big[ \ell\big( f_{\hat{\btheta}^\dagger}(\bv x),y \big) \Big] \leq
		\frac{1}{N} \sum_{n = 1}^{N} \ell\big( f_{\hat{\btheta}^\dagger}(\bx_{n}), y_{n} \big)
		+ \zeta_N(N) \leq \epsilon + \zeta_N(N)
		\text{.}
\end{equation}
Combining~\eqref{E:gap_bound}, \eqref{E:vc_F_bound}, and~\eqref{E:constraint_bound} using the union bound concludes the proof.
\end{proof}

\newpage
\section{Deriving the Langevin Monte Carlo sampler} \label{app:sampler}

In this appendix, we offer a more detailed derivation of the Langevin Monte Carlo sampler used in Algorithm~\ref{L:algorithm}.  Along the way, we present a brief, expository introduction to Hamiltonian Monte Carlo to provide the reader with further context concerning the derivation of Algorithm~\ref{L:algorithm}.  Much of this material is based on the derivations provided in standard references, including~\cite{betancourt2017conceptual,bishop2006pattern,neal2011mcmc}; we refer the reader to these references for a more complete treatment of these topics.

In the setting of our paper, given a fixed data point $(\bv x,y)\in\Omega$, our goal in deriving the sampler is to evaluate the following expectation:
\begin{align}
    \E_{\bdelta\sim\lambda(\bdelta|\bv x,y)} \left[ \ell(f_{\btheta}(\bv x + \bdelta), y)\right] = \int_{\Delta} \ell(f_{\btheta}(\bv x+\bdelta),y) \lambda(\bdelta|\bv x,y)
\end{align}
where $\lambda$ denotes the perturbation distribution defined by
\begin{align}
    \lambda(\bdelta|\bv x,y) = \frac{\ell(f_{\btheta}(\bv x + \bdelta), y)}{\int_\Delta \ell(f_{\btheta}(\bv x + \bdelta), y)d\bdelta} \label{eq:lmc-dist}
\end{align}
and where $f_{\btheta}\in\calF$ is a fixed classifier.  Roughly speaking, this problem is challenging due to the fact that we cannot compute the normalization constant in~\eqref{eq:lmc-dist}.  Therefore, although the form of~\eqref{eq:lmc-dist} indicates that the amount of mass placed on $\bdelta\in\Delta$ will be proportional to the loss $\ell(f_{\btheta}(\bv x, y))$ when the data is perturbed by this perturbation $\bdelta$, it's unclear how we can sample from this distribution in practice.

The first step in deriving the sampler is to introduce a \emph{momentum} variable $\bp$ to complement the space of perturbations: $\bdelta \to (\bdelta, \bp)$.  This transformation expands the $d$-dimensional perturbation space to a $2d$-dimensional \emph{phase space}.  Furthermore, this augmentation facilitates the lifting of $\lambda$ onto the so-called \emph{canonical distribution} $\blambda(\bdelta, \bp)$ defined by
\begin{align}
    \blambda(\bdelta, \bp) = \blambda(\bp|\bdelta) \cdot \blambda(\bdelta), \label{eq:canonical-dist}
\end{align}
which takes support over the $2d$-dimensional phase space.  Notably, as we have artificially introduced the momentum parameters $\bp$, the canonical density does not depend on a particular choice of parameterization of~\eqref{eq:canonical-dist}, and we can therefore express it in terms of an invariant \emph{Hamiltonian function} $H(\bdelta,\bp)$ defined by
\begin{align}
    \lambda(\bdelta,\bp) = \exp\left\{-H(\bdelta, \bp)\right\}, \quad\text{or equivalently}\quad H(\bdelta,\bp) = -\log \lambda(\bdelta,\bp).
\end{align}
Now, owing to the decomposition in~\eqref{eq:canonical-dist}, note that the Hamiltonian can be written as follows:
\begin{align}
    H(\bdelta,\bp) &= -\log \blambda(\bdelta|\bp) - \log \blambda(\bp) \\
    &\equiv K(\bdelta, \bp) + U(\bp).
\end{align}
where we have defined a \emph{kinetic energy} term $K(\bdelta,\bp) = -\log \blambda(\bdelta|\bp)$ as well as a \emph{potential energy} term $U(\bp) = - \log \blambda(\bp)$.  By evolving the parameters $(\bdelta,\bp)$ in phase space according to Hamilton's equations
\begin{align}
    \frac{d\bdelta}{dt} = + \frac{\partial H}{\partial \bp} = \frac{\partial K}{\partial \bp} \quad\text{and}\quad \frac{d\bp}{dt} = - \frac{\partial H}{\partial\bdelta} = - \frac{\partial K}{\partial \bdelta} - \frac{\partial V}{\partial \bdelta},
\end{align}
we generate a trajectory $\bdelta(t)$ that walks along the so-called \emph{typical set} of the perturbation distribution $\blambda(\bdelta)$ from which we want to sample.  Thus, to generate such a trajectory, we first choose a distribution for the momentum parameters $\bp$ and the integrate Hamilton's equations over time.

As is common in the literature, the next step in deriving the sampler is to place a probabilistic prior over $\bp$.  In what follows, we describe the samplers that result from two different priors.\\

\noindent\textbf{Gaussian prior.}  As is common in the literature, one can take the prior over $\bp$ to be a normal.  That is, we can let $\bp\sim\mathcal{N}(0, TI_d)$ where $T>0$ is a constant and $I_d$ is the $d$-dimensional identity matrix.  This in turn engenders a kinetic energy term $K(\bdelta, \bp) \propto (2T)^{-1}\norm{\bp}_2^2$.
Then, to (approximately) integrate Hamilton's equations, we employ the following \emph{leapfrog integration} update scheme:
\begin{align}
    \bdelta \gets \bdelta + \eta \nabla_{\bdelta} U(\bdelta) + \sqrt{2\eta T}\bp \quad\text{where}\quad \bp \sim\mathcal{N}(0, TI_d).
\end{align}

\noindent\textbf{Laplacian prior.}  Another common choice is to take the prior over $\bp$ to be Laplacian so that $\bp\sim \text{Laplacian}(0, T^2)$.  A similar calculation to the one performed for the Gaussian prior reveals that this implies that $K(\bdelta,\bp) \propto \frac{1}{T} \norm{\bp}_1$.  Integrating Hamiltonian's equations for this choice of the kinetic energy function yields the following scheme:
\begin{align}
    \bdelta\gets\bdelta + \eta\sign\left[ \nabla_{\bdelta} U(\bdelta) + \sqrt{2\eta T} \xi\right] \quad\text{where}\quad \xi\sim\text{Laplace}(0, T).
\end{align}

\newpage
\section{Further experimental details} \label{sect:hyperparams}

In this appendix, we provide further experimental details beyond those given in the main text.  All experiments were run across twelve NVIDIA RTX 5000 GPUs.  

\begin{table}[]
    \centering
    \caption{\textbf{Public implementations of baseline methods.}  In this table, we list the public implementations of popular adversarial training methods that we used to train baseline classifiers.}\vspace{5pt}
    \label{tab:impl}
    \begin{tabular}{cc} \toprule
        \textbf{Algorithm} & \textbf{Implementation} \\ \midrule
         PGD & \url{https://github.com/MadryLab/robustness} \\
         TRADES & \url{https://github.com/yaodongyu/TRADES} \\ 
         MART & \url{https://github.com/YisenWang/MART} \\ \bottomrule
    \end{tabular}

\end{table}

\paragraph{Training hyperparameters and data loading.}  We record the hyperparameters used for training the neural networks in Section~\ref{sect:experiments} on MNIST and CIFAR-10 below.

\begin{itemize}
    \item \textbf{MNIST.}  We train CNNs with two convolutional layers and two feed-forward layers.  In particular, we use the architecture from the MNIST PyTorch tutorial; the full architecture is described in the following file: \url{https://github.com/pytorch/examples/blob/master/mnist/main.py}.  We use a batch size of 128.  All adversarial perturbations are defined over the perturbation set $\Delta = \{\bdelta\in\R^d : \norm{\bdelta}_\infty \leq 0.3\}$.  All models were trained for 50 epochs with the Adadelta optimizer, and we used a learning rate of 1.0.
    \item \textbf{CIFAR-10.}  We train ResNet-18 and ResNet-50 classifiers with SGD and an initial learning rate of $0.01$.  We use 0.9 for the momentum, and we use weight decay with a penalty weight of $3.5 \times 10^{-3}$.  We train all classifiers for 200 epochs, and we decay the learning rate by a factor of 10 at epochs 150, 175, and 190.  In general, this is longer than CIFAR-10 classifiers are generally trained.  We increased the number of epochs to allow the dual variable to converge before the first learning rate step.  For completeness, we ran all baselines using the more standard training scheme of 120 total epochs with learning rate decays after epochs 55, 75, and 90; we noticed almost no difference in the final performance of the baselines for this shorter schedule.  We also apply random crops and random horizontal flips to the training data.  All adversarial perturbations are defined over the perturbation set $\Delta = \{\bdelta\in\R^d : \norm{\bdelta}_\infty \leq 8/255\}$.
\end{itemize}

\paragraph{Baseline implementations.}  As mentioned in the main text, we reran all baselines by adapting implementations released in prior work.  In particular, our implementations of the baseline methods are based on the public implementations recorded in Table~\ref{tab:impl}.  These methods are all implemented in our repository, which is publicly available at the following link: \url{https://github.com/arobey1/advbench}.

\paragraph{Baseline hyperparameters.}  Throughout the experiments section, we trained numerous baseline classifiers to offer points of comparison to our methods.  In this section, we list the hyperparameters used for each of these methods:
\begin{itemize}
    \item \textbf{PGD.}  On MNIST, we used 7 projected gradient ascent steps with a step size of 0.1.  On CIFAR-10, unless otherwise stated, we used 10 projected gradient ascent steps with a step size of 2/255.  Note that in Table~\ref{tab:num-steps}, we varied the number of ascent steps for PGD.
    \item \textbf{CLP \& ALP.}  The same step sizes and number of ascent steps were used for CLP and ALP as we reported above for PGD.  In line with~\cite{kannan2018adversarial}, we used a trade-off weight of $\lambda=1.0$ for both methods on MNIST and CIFAR-10.
    \item \textbf{TRADES.}  The same step sizes and number of ascent steps were used for TRADES as we reported above for PGD.  Following~\cite{zhang2019theoretically}, we used a trade-of weight of $\beta = 1/\lambda = 6.0$ for both datasets.
    \item \textbf{MART.}  The same step sizes and number of ascent steps were used for MART as we reported above for PGD. Following~\cite{wang2019improving}, we used a trade-off weight of $\lambda = 5.0$ for both datasets.
\end{itemize}

\paragraph{DALE hyperparameters.}  Unlike methods many of the baselines described above, DALE does not require the user to manually tune a weight which controls the trade-off between multiple objectives.  Instead, we use a primal-dual scheme to dynamically and adaptively update the weight on the clean objective.  Below, we provide some discussion of the hyperparameters inherent to our primal-dual approach.

\begin{itemize}
    \item \textbf{Margin} $\bm\rho$.  For MNIST, we found that a margin of 0.1
    \item \textbf{Dual step size} $\bm{\eta_d}$.  We found that the dual step size should be chosen to be significantly smaller than the primal step size.  By sweeping over $\eta_d\in\{0.1, 0.05, 0.01, 0.005, 0.005, 0.0001, 0.0005\}$, we found that a dual step size of $\eta_d = 0.001$ worked well in practice for CIFAR-10.
    \item \textbf{Primal step size $\bm{\eta_p}$.}  As described at the beginning of this appendix, we used $\eta_p = 1.0$ for MNIST and $\eta_p = 0.01$ for CIFAR-10.
    \item \textbf{Temperature $\mathbf{T}$.}  In practice, we found that the temperature should be chosen so that the noise coefficient $\sqrt{2\eta T}$ is relatively small.  By sweeping over $\sqrt{2\eta T}\in\{10^{-1}, 10^{-2}, 10^{-3}, 10^{-4}, 10^{-5}, 10^{-6}\}$, we found that robust performance began to degrade for $\sqrt{2\eta T} > 10^{-4}$.  For MNIST, we found that $\sqrt{2\eta T} = 10^{-3}$ worked well; on CIFAR-10, we used $\sqrt{2\eta T} = 10^{-4}$.
\end{itemize}

\newpage
\section{Further experiments}\label{app:further-exps}

\subsection{MNIST}

We first consider the MNIST dataset~\cite{MNISTWebPage}.  All models use a four-layer CNN architecture trained using the Adadelta optimizer~\cite{zeiler2012adadelta}.  To evaluate the robust performance of trained models, we report the test accuracy with respect to two independent adversaries.  In particular, we use a 1-step and a 10-step PGD adversary to evaluate robust performance; we denote these adversaries by FGSM and PGD$^{10}$ respectively.  

A summary of the performance of DALE and various state-of-the-art baselines is shown in Table~\ref{tab:mnist-linf}.  Notice that DALE marginally outperforms each of the baselines in robust accuracy, while maintaining a clean accuracy that is similar to that of ERM.  This indicates that on MNIST, DALE is able to reach high robust accuracies without trading off in nominal performance.  This table also shows a runtime analysis of each of the methods.  Notably, DALE and TRADES have similar running times, which is likely due to the fact in our implementation of DALE, we use the same KL-divergence loss to search for challenging perturbations.

\begin{table}
\centering
\caption{\textbf{Accuracy and computational complexity on MNIST.}  In this table, we report the test accuracy and computational complexity of our method and various state-of-the-art baselines on MNIST.  In particular, all methods are trained using a four-layer CNN architecture, and we use a perturbation set of $\Delta = \{\bdelta\in\R^d : \norm{\bdelta}_\infty \leq 0.3\}$.  Our results are highlighted in gray.}\vspace{5pt}
\label{tab:mnist-linf}
\begin{tabular}{cccccccc}
\toprule
& & \multicolumn{3}{c}{Test accuracy (\%)} & \multicolumn{2}{c}{Performance (sec.)} \\ \cmidrule(lr){3-5} \cmidrule(lr){6-7}
\textbf{Algorithm} & $\bm\rho$ & \textbf{Clean} & \textbf{FGSM} & \textbf{PGD$^{10}$} & \textbf{Batch} & \textbf{Epoch} \\
\midrule
ERM & - & 99.3 & 14.3 & 1.46 & 0.007 & 3.47 \\ \midrule
FGSM & - & 98.3 & 98.1 & 13.0 & 0.011 & 5.48  \\
PGD & - & 98.1 & 95.5 & 93.1 & 0.039 & 18.2 \\
CLP & - & 98.0 & 95.4 & 92.2 & 0.047 & 21.9 \\
ALP & - & 98.1 & 95.5 & 92.5 & 0.048 & 22.0 \\
TRADES & - & 98.9 & 96.5 & 94.0 & 0.055 & 25.8 \\
MART & - & 98.9 & 96.1 & 93.5 & 0.043 & 20.4 \\
\midrule
\rowcolor{Gray} DALE & 1.0 & 99.1 & 97.7 & 94.5 & 0.053 & 25.4 \\
\bottomrule
\end{tabular} 
\end{table}

\subsection{CIFAR-10}

\begin{table}[]
    \centering
    \caption{\textbf{Accuracy and computational complexity on CIFAR-10.}  In this table, we report the test accuracy and computational complexity of our method and various state-of-the-art baselines on CIFAR-10.  In particular, all methods are trained using a ResNet-18 architecture, and we use a perturbation set of $\Delta = \{\bdelta\in\R^d : \norm{\bdelta}_\infty \leq 8/255\}$.  Our results are highlighted in gray.  Of note is the fact that our method advances the state-of-the-art both in adversarial and in clean accuracy.}\vspace{5pt}
    \label{tab:cifar-linf}
    \begin{tabular}{ccccccc} \toprule
        & & \multicolumn{3}{c}{Test accuracy (\%)} & \multicolumn{2}{c}{Performance (sec.)} \\ \cmidrule(lr){3-5} \cmidrule(lr){6-7}
         \textbf{Algorithm} & $\bm{\rho}$ & \textbf{Clean} & \textbf{FGSM} & \textbf{PGD$^{20}$} & \textbf{Batch} & \textbf{Epoch} \\ \midrule
         ERM & - & 94.0 & 0.01 & 0.01 & 0.073 & 28.1  \\ \midrule
         FGSM & - & 72.6 & 49.7 & 40.7 & 0.135 & 53.0 \\
         PGD & - & 83.8 & 53.7 & 48.1 & 0.735 & 287.9 \\
         CLP & - & 79.8 & 53.9 & 48.4 & 0.872 & 340.5 \\
         ALP & - & 75.9 & 55.0 & 48.8 & 0.873 & 341.2 \\
         TRADES & - & 80.7 & 55.2 & 49.6 & 1.081 & 422.0 \\
         MART & - & 78.9 & 55.6 & 49.8 & 0.805 & 314.1 \\ \midrule
         \rowcolor{Gray} DALE & 0.5 & \textbf{86.0} & 54.4 & 48.4 & 1.097 & 421.4 \\
         \rowcolor{Gray} DALE & 0.8 & 85.0 & 55.4 & 50.1 & 1.098 & 422.6 \\
         \rowcolor{Gray} DALE & 1.1 & 82.1 & 55.2 & \textbf{51.7} & 1.097 & 421.0 \\
         \bottomrule
    \end{tabular}
\end{table}

We next consider the CIFAR10 dataset~\cite{krizhevsky2009learning}.  Throughout this section, we use the ResNet-18 architecture trained using SGD, and we consider adversaries which can can generate perturbations $\bdelta$ lying within the perturbation set $\Delta = \{\bdelta\in\R^d : \norm{\bdelta}_\infty\leq 8/255\}$.  To this end, we use evaluate the robust performance of trained models using FGSM and PGD$^{20}$ adversaries.  For all of the classifiers trained using DALE, we use the KL-divergence loss for $\ell_\text{pert}$ and $\ell_\text{ro}$, and we use the cross-entropy loss for $\ell_\text{nom}$.

In Table~\ref{tab:cifar-linf}, we show a summary of our results on CIFAR-10.  One notable aspect of our results is that DALE trained with $\rho=0.8$ is the only model to achieve greater than 85\% clean accuracy and greater than 50\% robust accuracy.  This indicates that DALE is more successfully able to mitigate the trade-off between robustness and nominal performance.  And indeed, the baselines that have relatively high robust accuracy (TRADES and MART) suffer a significant drop in clean accuracy relative to DALE (-4.3\% for TRADES and -6.1\% for MART when compared with DALE trained with $\rho=0.8$).  Table~\ref{tab:cifar-linf} also shows a comparison of the computation time for each of the methods.  These results indicate that the computational complexity of DALE is on a par with TRADES.

\paragraph{A closer look at the trade-off between accuracy and robustness.}

We next study the trade-off between robustness and nominal performance of DALE for two separate architectures: ResNet-18 and ResNet-50.  In our formulation, the parameter $\rho$ explicitly captures this trade-off in the sense that a smaller $\rho$ will require a higher level of nominal performance, which in turn reduces the size of the feasible set.  This reduction has the effect of limiting the robust performance of the classifier.

In Table~\ref{tab:cifar-rho-trade-off}, we illustrate this trade-off by varying $\rho$ from 0.1 to 1.1.  For both architectures, the trade-off is clearly reflected in the fact that increasing the margin $\rho$ has the simultaneous effect of decreasing the clean accuracy and increasing the robust accuracy for both adversaries.  We highlight that for the ResNet-18 architecture, when the constraint is enforced with a relatively large margin (e.g., $\rho\geq 1.0$), DALE achieves nearly 52\% robust accuracy against \text{PGD}$^{20}$, which is nearly two percentage points higher than any of the baseline classifiers in Table~\ref{tab:mnist-and-cifar-linf}.  On the other hand, when the margin is relatively small (e.g., $\rho\leq 0.2)$, there is almost no trade-off in the clean accuracy relative to ERM in Table~\ref{tab:mnist-and-cifar-linf}, although as a result of this small margin, the robust performance takes a significant hit.  Interestingly, with regard to the classifiers trained using ResNet-50, it seems to be the case that the margin $\rho$ corresponding to the largest robust accuracy is different than the peak for ResNet-18.  

\begin{table}
    \centering
    \caption{\textbf{Evaluating the trade-off between robustness and accuracy.}  To evaluate the trade-off between robustness and nominal performance, we train ResNet-18 and ResNet-50 models on CIFAR-10 for different trade-off parameters $\rho$.  Notice that across both architectures, the impact of increasing $\rho$ is to simultaneously decrease clean performance and increase robust performance.}\vspace{5pt}
    \label{tab:cifar-rho-trade-off}
    \begin{tabular}{ccccccc} \toprule
         & \multicolumn{3}{c}{ResNet-18} & \multicolumn{3}{c}{Resnet-50}\\ \cmidrule(lr){2-4} \cmidrule(lr){5-7}
         $\bm\rho$ & \textbf{Clean} & \textbf{FGSM} & \textbf{PGD$^{\mathbf{20}}$} & \textbf{Clean} & \textbf{FGSM} & \textbf{PGD$^{\mathbf{20}}$} \\ \midrule
         0.1 & 93.0 & 35.6 & 1.50 & 93.8 & 23.9 & 16.7 \\
         0.2 & 92.4 & 43.6 & 11.9 & 93.7 & 20.5 & 16.3 \\
         0.3 & 88.7 & 42.4 & 31.2 & 90.1 & 43.0 & 24.8 \\
         0.4 & 86.4 & 50.9 & 44.3 & 86.2 & 50.5 & 38.4 \\
         0.5 & 86.0 & 54.4 & 48.4 & 86.5 & 50.1 & 42.6 \\
         0.6 & 85.6 & 54.6 & 49.0 & 86.1 & 57.7 & 52.0 \\
         0.7 & 85.3 & 56.2 & 50.3 & 84.7 & 57.0 & 51.4 \\
         0.8 & 83.8 & 55.4 & 50.1 & 84.3 & 56.4 & 50.8 \\
         0.9 & 83.8 & 56.0 & 51.3 & 83.9 & 55.9 & 51.2 \\ 
         1.0 & 82.2 & 54.7 & 51.2 & 82.1 & 54.2 & 50.1 \\ 
         1.1 & 82.1 & 55.2 & 51.7 & 80.4 & 52.3 & 49.9 \\ \bottomrule
    \end{tabular}
\end{table}

\paragraph{Impact of the number of Langevin iterations.}  In Table~\ref{tab:num-steps}, we study the impact of varying the number of Langevin iterations $L$.  For each row in this table, we train a ResNet-18 classifier with $\rho=1.0$; as before, we use the KL-divergence loss for $\ell_\text{pert}$ and $\ell_\text{ro}$, and we use the cross-entropy loss for $\ell_\text{nom}$.  As one would expect, when $L$ is small, the trained classifiers have relatively high clean accuracy and relatively low robust accuracy.  To this end, increasing $L$ has the simultaneous effect of decreasing clean accuracy and increasing robust accuracy. 

To offer a point of comparison, we also show the analogous results for PGD run using the cross-entropy loss, where $L$ is taken to be the number of steps of projected gradient ascent.  As each Langevin iteration of DALE effectively amounts to a step of projected gradient ascent with noise, we expect that the impact of varying $L$ in DALE will be analogous to the impact of varying the number of training-time PGD steps.  And indeed, as we increase $L$, the robust performance of PGD improves and the clean performance decreases.

\begin{table}[]
    \centering
    \caption{\textbf{Impact of the number of ascent steps.}  In this table, we show the impact of varying the number of Langevin steps used by DALE in lines~4-7 of Algorithm~\ref{L:algorithm}.  To offer a point of comparison, we also show the impact of varying the number of ascent steps for PGD.}\vspace{5pt}
    \label{tab:num-steps}
    \begin{tabular}{ccccccc} \toprule
         & \multicolumn{3}{c}{PGD$^{L}$} & \multicolumn{3}{c}{DALE} \\ \cmidrule(lr){2-4} \cmidrule(lr){5-7}
         $\mathbf{L}$ & \textbf{Clean} & \textbf{FGSM} & \textbf{PGD$^{\mathbf{20}}$} & \textbf{Clean} & \textbf{FGSM} & \textbf{PGD$^{\mathbf{20}}$}\\ \midrule
         1 & 92.9 & 52.3 & 23.7 & 87.2 & 46.6 & 39.0 \\
         2 & 90.9 & 49.9 & 36.6 & 85.4 & 53.6 & 47.1 \\
         3 & 87.7 & 50.6 & 41.5 & 84.0 & 55.0 & 50.2 \\
         4 & 84.5 & 52.2 & 43.3 & 82.8 & 55.0 & 50.7 \\
         5 & 83.6 & 53.5 & 47.9 & 82.5 & 54.9 & 50.7 \\
         10 & 83.8 & 53.7 & 48.1 & 82.2 & 54.7 & 51.2 \\
         15 & 82.9 & 54.0 & 48.0 & 81.0 & 54.7 & 51.0 \\
         20 & 83.0 & 54.4 & 48.3 & 81.0 & 54.7 & 51.4 \\ \bottomrule
    \end{tabular}
\end{table}

\newpage
\section{On the convergence of Algorithm~\ref{L:algorithm}}\label{app:conv}

Observe that Algorithm~\ref{L:algorithm} is a primal-dual algorithm~\cite{bubeck2014convex} in which the sampling procedure in steps~3--7 is used to obtain an estimate of the stochastic gradient of the primal problem. When~$\btheta \mapsto \ell\big( f_{\btheta}(\cdot), \cdot \big)$ is convex~(e.g., for linear, kernel, or logistic models), it is well-known that SGD converges almost surely as long as this gradient estimate is unbiased~\cite{bonnans2019convex}. As is typical with LMC, we omitted the Metropolis-Hastings acceptance step in Algorithm~\ref{L:algorithm} that would guarantee unbiased estimates~\cite{neal2011mcmc}. Still, when~$g$ is log-concave~(e.g., the softmax output of a CNN), this procedure approaches the true distribution in total variation norm, which implies that its bias can be made arbitrarily small~\cite{bubeck2015finite}. This is enough to guarantee almost sure convergence to a neighborhood of the optimum~\cite{bertsekas2000gradient,ajalloeian2020analysis}.

The convergence properties of primal-dual methods are less well understood when~$\btheta \mapsto \ell\big( f_{\btheta}(\cdot), \cdot \big)$ is non-convex. Nevertheless, a good estimate of the primal minimizer is enough to obtain an approximate gradient for dual ascent~\cite{paternain2019constrained, chamon2020probably}. There is overwhelming empirical and theoretical evidence that this is the case for overparametrized models, such as CNNs, trained using gradient descent~\cite{soltanolkotabi2018theoretical, zhang2016understanding, arpit2017closer, ge2017learning, brutzkus2017globally}. We can then run the primal~(step~8) and dual~(step~10) updates at different timescales so as to obtain a good estimate of the primal minimizer before performing dual ascent.
\newpage
\section{Further related work}
\label{S:related_work}

\paragraph{Adversarial robustness.}

As described in Section~\ref{sec:intro}, it is well-know that state-of-the-art classifiers are susceptible to adversarial attacks~\cite{biggio2013evasion, carlini2017towards, hendrycks2019benchmarking, djolonga2020robustness, taori2020measuring, hendrycks2020many, torralba2011unbiased, goodfellow2014explaining}.  Toward addressing this challenging, a rapidly-growing body of work has provided \emph{attack algorithms} to generate data perturbations that fool classifiers and \emph{defense algorithms} which are designed to train robust classifiers to be robust against these perturbations. However, despite the myriad of work in this field and significant improvements on a number of well-known benchmarks~\cite{sinha2017certifying, gao2017wasserstein, ben2009robust,salman2019provably, cohen2019certified, kumar2020curse,madry2017towards, wong2017provable, huang2017adversarial, sinha2018gradient, shaham2018understanding},  there are still many open questions on when adversarial learning is even possible and in what sense~\cite{awasthi2020adversarial, yin2019rademacher, cullina2018pac, montasser2020efficiently, montasser2019vc}.  Unlike the majority of these works, we exploit duality to derive a principled primal-dual style algorithm from first principles for the adversarial robustness setting.

\paragraph{Constrained optimization.}  Also related are works that seek to enforce constraints on learning problems~\cite{donti2021dc3}.  While several heuristic algorithms exist for this setting, many focus on restricted classes of constraints \cite{pathak2015constrained,chen2018approximating,frerix2020homogeneous,amos2017optnet,ravi2018constrained} and those that can handle more general constraints come at the cost of added computation complexity \cite{agrawal2019differentiable,karras1995efficient}.  Moreover, each of these works seeks to enforce constraints on a particular parameterization for the learning problem (such as directly on the weights of a neural network) rather than on the underlying statistical problem, as we do in this paper.  In this way, our work is more related to the primal-dual style algorithms which often arise in convex optimization~\cite{bubeck2014convex,chamon2020empirical}.

\end{document}